\def\cleanbuild{1}
\documentclass[11pt]{article}

\usepackage[margin=1in]{geometry}
\usepackage[T1]{fontenc}
\usepackage{libertinus}             
\usepackage{microtype}
\usepackage{amsmath, amssymb, amsthm}
\usepackage{mathtools}
\usepackage[libertine]{newtxmath}   
\usepackage{natbib}
\usepackage{booktabs}
\usepackage{tikz,pgfplots}
\pgfplotsset{compat=1.18}
\usepackage{xcolor}
\definecolor{linkcol}{rgb}{0,0,0.55}      
\definecolor{citecol}{rgb}{0,0.45,0}
\definecolor{urlcol}{rgb}{0.55,0,0}
\usepackage{hyperref}

\newif\ifreviewannotations
\ifdefined\cleanbuild\reviewannotationsfalse\else\reviewannotationstrue\fi
\ifreviewannotations
\fi
\definecolor{rev-AM}{RGB}{200,80,40}    
\definecolor{rev-ZV}{RGB}{40,140,80}    
\definecolor{rev-SM}{RGB}{120,40,160}   
\definecolor{rev-PB}{RGB}{40,80,200}

\hypersetup{colorlinks=true, linkcolor=linkcol, citecolor=citecol, urlcolor=urlcol}

\newcommand{\R}{\mathbb{R}}
\newcommand{\norm}[1]{\left\lVert #1 \right\rVert}
\newcommand{\db}{d_{\mathcal{B}}}
\newcommand{\dmat}{d_{\mathrm{mat}}}
\newcommand{\dmatS}{d_{\mathrm{mat}}^{S}}
\newcommand{\D}[1]{\mathcal{D}_{#1}}
\newcommand{\LC}{\nu}                 
\newcommand{\netcfg}{\mathcal{C}}     
\newcommand{\Phii}{\Phi}              
\newcommand{\Psii}{\Psi}              
\newcommand{\rhonu}{\rho_-}           
\newcommand{\Hilb}{\mathcal{H}}       
\newcommand{\kappaminus}{\kappa_-}    
\newcommand{\hatrhonu}{\widehat{\rho}_-}
\newcommand{\Nmax}{N_{\max}}
\newcommand{\Sph}{\mathbb{S}^{d-1}_{>0}}   
\newcommand{\Kern}{\kappa}          
\newcommand{\Fan}{\mathcal{U}}      
\newcommand{\wmin}{w_{\min}}
\newcommand{\gtheta}{g_\theta}        
\newcommand{\ftheta}{f_\theta}        
\newcommand{\eps}{\varepsilon}
\DeclareMathOperator*{\argmax}{arg\,max}

\newtheorem{theorem}{Theorem}[section]
\newtheorem{proposition}[theorem]{Proposition}
\newtheorem{lemma}[theorem]{Lemma}
\newtheorem{corollary}[theorem]{Corollary}
\theoremstyle{definition}
\newtheorem{definition}[theorem]{Definition}

\theoremstyle{remark}
\newtheorem{remark}[theorem]{Remark}
\newcommand{\takeaway}[2]{\par\smallskip\noindent\emph{#2}\par\smallskip}

\title{COMPLEX: A Closed-Form Certified Embedding\\
of Multiparameter Persistence Modules}
\usepackage{authblk}

\author[1]{Sushovan Majhi}
\author[2]{Atish Mitra}
\author[3]{\v{Z}iga Virk}
\affil[1]{Data Science, George Washington University, USA \quad \texttt{s.majhi@gwu.edu}}
\affil[2]{Department of Mathematical Sciences, Montana Technological University, USA \quad \texttt{amitra@mtech.edu}}
\affil[3]{Faculty of Computer and Information Science, University of Ljubljana, and Institute IMFM, Slovenia \quad \texttt{ziga.virk@fri.uni-lj.si}}
\author[4]{Pramita Bagchi}
\affil[4]{Biostatistics and Bioinformatics, George Washington University, USA \quad \texttt{pramita.bagchi@gwu.edu}}
\date{}

\begin{document}
\maketitle

\begin{abstract}
Every multiparameter persistence vectorization we know of carries a one-sided Lipschitz \emph{upper} bound and nothing below it: without a lower gauge there is no sense in which the features are faithful, and no per-prediction guarantee can be built on them.
This paper supplies the missing side.
\textbf{COMPLEX} is a closed-form, training-free embedding of multiparameter modules---slice the module along a fixed near-diagonal net, embed each slice barcode by the certified PLACE/PALACE landmark map, concatenate.
Under a checkable \emph{witnessing-slice coherence} condition, holding on $100\%$ of audited pairs on Orbit5k, a single slice carries a closed-form lower gauge: separated modules stay separated in the embedding.
With the standard upper bound this gives, to our knowledge, the first two-sided distortion bound for a multiparameter feature map, making faithfulness \emph{measurable}.
Measuring it, we find the floor tight within a small factor of realized distances yet operationally \emph{local}: an RBF-SVM reaches $91\%$ where $1$-NN reaches $78\%$ on the same features.
Local per-prediction certification therefore fails for a structural reason common to every landmark embedding whose lower gauge is witnessed by one coordinate.
With no learned embedding and no held-out calibration---only a cross-validated SVM head---COMPLEX sets the state of the art on both Orbit benchmarks ($91.95\%$ on Orbit5k, $92.98\%$ on Orbit100k), level with or above Euler-characteristic surfaces and above transformers and graphcode.
On graphs it exceeds \textsc{Gril} on all four shared molecular benchmarks with one fixed configuration, including the only multiparameter method to clear COX2's majority baseline by more than three points.
Closed-form selection---of the landmark radius, the kernel (certificate-preserving), and the bifiltration set---buys further accuracy; gradient-shaped adaptation buys none.
\end{abstract}

\section{Introduction}\label{sec:intro}

Multiparameter persistence assigns a dataset not a single barcode but a persistence module over $\R^d$, arising for instance from a density--Rips or a sublevel bifiltration.
A growing literature turns such modules into features for learning: multiparameter persistence landscapes~\citep{Vipond2020} and images~\citep{Carriere2020-mpi}, signed measures and sliced barcodes~\citep{Botnan2022signed,Loiseaux2023}, the \textsc{Gril} landscape and its differentiable successor~\citep{Xin2023GRIL,Mukherjee2024}, and graphcodes~\citep{KerberRussold2024}.
Two further families compete on the same benchmarks without being persistence methods: Euler-characteristic curves and surfaces~\citep{HacquardLebovici2024} read the same bifiltration but bypass homology altogether, and transformer architectures~\citep{Reinauer2021,xPerT} read one-parameter diagrams directly.
These methods split along two orthogonal axes (Table~\ref{tab:landscape}).
On \emph{cost}, some are training-free vectorizations---landscapes, images, signed measures, Euler surfaces---and others are trained---\textsc{Gril}, its differentiable successor D-\textsc{Gril}, and transformers.
On \emph{guarantee}, they are unanimous: each carries a one-sided Lipschitz \emph{upper} bound (small module distance implies small feature distance) but never the converse two-sided \emph{lower} bound with explicit constants that a separation guarantee requires.
To our knowledge the certified column is empty: no multiparameter-persistence feature map supplies a computable lower distortion bound, and hence none yields a split-free per-prediction certificate derived from the feature geometry itself.
(Conformal prediction~\citep{Angelopoulos2023} can be wrapped around any of them, but it guarantees marginal coverage of the true label under exchangeability from a calibration split, whereas the certificate here is deterministic, split-free and per-instance and guarantees stability rather than correctness; the two are complementary, as Section~\ref{sec:related} spells out.)
This paper fills that column in both rows, with one construction and one guarantee.

\begin{table}[t]
\centering
\caption{The landscape of topological features computed from a multiparameter filtration (ECS reads the filtration but computes no persistence). Every prior method is uncertified; this paper supplies the certified version of \emph{both} regimes.}
\label{tab:landscape}
\begin{tabular}{lll}
\toprule
 & uncertified & \textbf{certified (this paper)} \\
\midrule
training-free & mp-landscape/image, signed measures, ECS & \textbf{COMPLEX} \\
trained       & \textsc{Gril}, D-\textsc{Gril}, graphcode, Persformer & \textbf{D-COMPLEX} \\
\bottomrule
\end{tabular}
\end{table}

In one parameter, PLACE~\citep{PaperI} and its data-adaptive successor PALACE~\citep{PaperII} already occupy the certified, training-free cell: a closed-form landmark embedding of persistence diagrams with a lower distortion certificate $\rhonu$ derived from labels alone, no tuning.
But they certify only \emph{one} parameter.
This paper carries the certificate into multiparameter modules, then spends its second half measuring what the certificate is worth; a differentiable variant (Section~\ref{sec:dcomplex}) shows it survives gradient adaptation as well.
The mechanism is the \emph{slice-stack}: embed a module by stacking the per-slice diagram embeddings over a finite slice net, and certify the \emph{sliced matching distance} $\dmatS$.
We call the construction \textbf{COMPLEX}---\emph{Certified One-slice Multiparameter Persistence Landmark EXtension}---``one-slice'' because a single witnessing slice carries the certificate to the module level (Theorem~\ref{thm:cert}).
This paper develops it self-containedly on top of PLACE/PALACE~\citep{PaperI,PaperII}, whose population-level statistics---covariance estimation, two-sample tests and confidence sets for the mean embedding---are developed in a companion paper~\citep{PaperIII}; the asymptotic, coarse-geometric behavior of the same construction on the unbounded module space is a separate study we defer to future work.
The whole single-parameter PLACE/PALACE apparatus---landmarks, the certificate $\rhonu$, $\LC$-coherence, the per-prediction guarantee---transfers slice by slice; the only new objects are the slice net and the admissibility-at-realized-scale it forces (Remark~\ref{rem:angle}).

\subsection{Contribution}\label{sec:contrib}
\begin{enumerate}\itemsep2pt
\item[\emph{(i)}] \emph{A two-sided bound.} The slice-stack $\Psii$ (Definition~\ref{def:embed}) is $1$-Lipschitz up to $\sqrt S\,\Nmax$ for $\dmatS$ (Proposition~\ref{prop:upper}) and, under a checkable \emph{witnessing-slice coherence} condition (Definition~\ref{def:coh}), obeys the closed-form lower gauge $\norm{\Psii(M)-\Psii(M')}\ge\rhonu(\tau)=\tfrac14\tau\,\wmin$ (Theorem~\ref{thm:cert}): separated modules stay separated.
  Coherence needs no multiparameter matching, has a matching-free sufficient condition (Proposition~\ref{prop:suff}), and holds on $100\%$ of audited cross-class pairs at the headline configuration (Section~\ref{sec:exp}).
\item[\emph{(ii)}] \emph{The net.} The certificate's value is angle-independent but its price is not (Remark~\ref{rem:angle}), which makes the fixed near-diagonal fan the cheapest certifiable net; a certified data-adaptive net generalizes it (Section~\ref{sec:adaptive}).
  No slice is ever optimized by gradient or held-out search.
\item[\emph{(iii)}] \emph{Predictions.} A certified nearest-centroid rule inherits PALACE's split-free per-prediction guarantee (Theorem~\ref{thm:certified}), but that guarantee holds for any bounded embedding.
  The rule that reads $\rhonu$ is nearest neighbor (Corollary~\ref{cor:nn}), which returns with each accepted label a stability radius in the module metric.
  The ratio $B/\rhonu$ controls a---numerically vacuous---generalization bound (Proposition~\ref{prop:gen}).
\item[\emph{(iv)}] \emph{Adaptation keeps the certificate.} The floor constrains the output embedding, so the bifiltration can be selected in closed form or learned by gradient (D-COMPLEX, with an exact gather rule, Proposition~\ref{prop:grad}), and the kernel head chosen among Gaussian, $\chi^2$ and Hellinger (Lemma~\ref{lem:kersel}), without losing $\rhonu$.
\item[\emph{(v)}] \emph{Faithfulness, measured.} The floor is tight within a small factor of realized distances, yet every per-prediction rule built on it accepts under $1.5\%$ of test modules at the budgets that maximize accuracy.
  One mechanism explains it: on the same embedding an RBF-SVM reaches $91\%$ where $1$-NN reaches $78\%$---the accuracy is \emph{global}, carried by the kernel, while $\rhonu$ certifies a \emph{local} separation in classes interleaved at neighbor scale.
  Five local repairs fail at the same wall.
  The global one---a guarantee attached to the kernel---is vacuous because the fitted decision function is a high-gain map of the embedding; a guarantee for the accurate head, if one exists, is not a Lipschitz one (Section~\ref{sec:exp}).
  None of this is specific to persistence: it applies to any landmark embedding whose lower gauge is witnessed by a single coordinate.
\item[\emph{(vi)}] \emph{Accuracy.} With no learned embedding and no held-out calibration (the SVM head is cross-validated), COMPLEX sets the state of the art on both Orbit benchmarks---$91.95\%$ on Orbit5k, $92.98\%$ on Orbit100k---level with or above Euler-characteristic surfaces and above transformers and the GNN-trained graphcode.
  The full Orbit5k experiment runs in $6.8$ minutes on $16$ CPU cores.
  On graphs it exceeds \textsc{Gril} on all four shared molecular sets with one fixed configuration.
  Closed-form selection over larger spaces---radius, kernel, bifiltration set---buys further accuracy; gradient-shaped adaptation buys none (Sections~\ref{sec:exp}--\ref{sec:dcomplex}).
\end{enumerate}

\subsection{Related work}\label{sec:related}

\paragraph{Multiparameter vectorizations.}
Most feature maps for multiparameter modules read the fibered barcode and vectorize it: multiparameter persistence landscapes~\citep{Vipond2020} and images~\citep{Carriere2020-mpi} sample a piecewise-linear or convolved summary over a line net; signed barcodes and Hilbert/Euler signed measures~\citep{Botnan2022signed, Loiseaux2023,LoiseauxCarriereBlumberg2023} compare modules through sliced-Wasserstein kernels.
All are $1$-Lipschitz \emph{upper} bounds---small module distance implies small feature distance---but none supplies the converse lower bound that a separation guarantee needs.
COMPLEX reuses the same fibered-barcode input and the same near-diagonal line net, but embeds each slice by the closed-form PLACE/PALACE landmark map~\citep{PaperI,PaperII}, which \emph{does} carry a two-sided certificate, and lifts it to the module by a single witnessing slice.

\paragraph{Learned multiparameter features.}
\textsc{Gril}~\citep{Xin2023GRIL} generalizes the persistence landscape to a $2$-parameter rank-invariant landscape, and D-\textsc{Gril}~\citep{Mukherjee2024} makes the underlying bifiltration \emph{learned}, differentiating the landscape through a graph neural network.
Euler-characteristic surfaces with gradient-boosted trees~\citep{HacquardLebovici2024} bypass homology for a large speed gain, transformer architectures~\citep{Reinauer2021,xPerT} read the diagrams directly, and graphcodes~\citep{KerberRussold2024} feed a two-parameter hierarchy of barcodes to a graph neural network, previously the strongest reproducible accuracy at the Orbit100k scale ($92.3$; Section~\ref{sec:exp}).
These are the trained or tuned methods COMPLEX aims to match with no learned embedding.
On graphs its natural comparison within the multiparameter family is \textsc{Gril}, which it exceeds on all four shared molecular sets (Section~\ref{sec:exp}) while additionally certifying its predictions---something a learned pipeline cannot express.
Beyond that family, trained topological pipelines that couple persistence with graph neural networks or attention---TopoGCL~\citep{ChenFriasGel2024} among them---report higher accuracy on several TU benchmarks.
Section~\ref{sec:exp} names them where they lead; the comparison this paper makes is with no learned embedding and no held-out calibration.

\paragraph{Certified single-parameter embeddings.}
PLACE~\citep{PaperI} and its data-adaptive successor PALACE~\citep{PaperII} give a closed-form landmark embedding of one-parameter diagrams with a label-derived lower distortion certificate and a per-prediction nearest-centroid guarantee.
Both are deterministic and split-free, and they certify the geometry of a prediction---its stability radius and its agreement with the population rule---rather than the marginal coverage of the true label that distribution-free conformal prediction~\citep{Angelopoulos2023} guarantees under exchangeability from a calibration split.
The two are complementary rather than substitutes: any certificate margin is itself a valid nonconformity score, so a conformal wrapper can be placed on top of the certified head without touching the certificate.
COMPLEX is their multiparameter lift: the per-slice certificate is PALACE's, and the new content is the slice-stack that carries it to a module through a single witnessing slice, certifying the sliced matching distance $\dmatS$.

\section{COMPLEX: the slice-stack and its two-sided bound}\label{sec:embed}

We fix the objects (Section~\ref{sec:bg}), recall the single-parameter PLACE/PALACE map and stack it over a net (Section~\ref{sec:embed-def}), and prove the two-sided bound: the upper bound there, the lower bound in Section~\ref{sec:cert} under the coherence condition of Section~\ref{sec:coh}; the net itself is fixed in Section~\ref{sec:slicenet}.

\subsection{Modules, slices, and the sliced matching distance}\label{sec:bg}

We recall finitely presented multiparameter modules, slicing, and the sliced matching distance; all are standard (see~\citep{CarlssonZomorodian2009-multidim,Lesnick2015} for the algebra).
Throughout, $\Bbbk$ is the coefficient field, $k$ indexes landmarks, $s$ indexes slices, $q$ indexes homological degree, and $\sigma$ denotes a simplex; slice angles are $\vartheta_s$, the FPS radius factor is $\alpha$, and network parameters are $\theta$; the RBF bandwidth is written $h$ and the kernel $\Kern$, whose cross-class ceiling is $\kappaminus$.
Table~\ref{tab:notation} collects them; the few symbols reused across scopes---$\tau$, $\alpha$, $d$---are disambiguated where they occur.
Each result is followed by a one-line takeaway in plain language; every result in this paper is proved, and the one whose number is uninformative says so in its takeaway.

A \emph{$d$-parameter persistence module} $M\colon\R^d_{\ge0}\to\mathrm{Vec}_{\Bbbk}$ is taken finitely presented (which also implies it is pointwise finite-dimensional).
Non-negative grades are what every filtration in this paper produces---scale, codensity, heat-kernel signature, degree and node label are all non-negative, and the axis calibration of Section~\ref{sec:exp} is a positive rescaling---so no module below is translated into the orthant at all; a finitely presented module over $\R^d$ has finitely many generator and relation grades, and what a translation would cost is recorded after \eqref{eq:slicepar}.
A \emph{slice} is a line $\ell\subset\R^d$ of positive slope; the restriction $M|_\ell$ is a one-parameter module with barcode $\mathrm{dgm}(M|_\ell)\in\D N$, and $\ell\mapsto\mathrm{dgm}(M|_\ell)$ is the \emph{fibered barcode}, equivalent to the rank invariant~\citep{CerriDiFabioFerriEtAl2013,LesnickWright2015}.
We write $\mathcal R_N$ for the modules with at most $N$ bars in every slice.

\paragraph{Slices, parametrized.}
Both the certificate and the implementation depend on \emph{how} a slice is parametrized, so we fix it.
Every coordinate is strictly increasing along a line of positive slope, so such a line enters $\R^d_{\ge0}$ at a single point $\beta$, and that point has $\min_i\beta_i=0$.
Writing the line as the ray from its entry point,
\begin{equation}\label{eq:slice}
  \ell_{\beta,u}=\{\beta+tu:\ t\ge0\},\qquad
  u\in\Sph:=\{u\in\R^d_{>0}:\norm u_2=1\},\quad
  \beta\in\R^d_{\ge0}\ \text{with}\ \min_i\beta_i=0,
\end{equation}
therefore names every slice exactly once, with $t$ the Euclidean \emph{arclength} from $\beta$.
A generator at grade $\gamma\in\R^d_{\ge0}$ enters $M|_{\ell_{\beta,u}}$ at
\begin{equation}\label{eq:slicepar}
  s_{\beta,u}(\gamma)\ =\ \min\{t\ge0:\ \beta+tu\ge\gamma\}\ =\ \max_{1\le i\le d}\ (\gamma_i-\beta_i)/u_i\ \ge\ 0 ,
\end{equation}
the sign because $\beta_j=0\le\gamma_j$ for some $j$, so slice barcodes land in $\D N$ with non-negative births as that space requires.
Both parameters earn their place.
A comparable pair $\gamma\le\gamma'$ lies on a common ray through the origin only when $\gamma'$ is a positive multiple of $\gamma$, so the directions alone read a one-parameter subfamily of the slices, and the offsets are what let the fibered barcode see the rest.  (A slice through both $\gamma$ and $\gamma'$ exists exactly when $\gamma'_i>\gamma_i$ in every coordinate; the equivalence with the full rank invariant extends to the remaining comparable pairs by constructibility~\citep{CerriDiFabioFerriEtAl2013,LesnickWright2015}.)
They also make the construction translation-covariant: translating a module and every base point of the net by the same $c\in\R^d$ leaves each slice barcode unchanged.  Two consequences, worth separating.  The continuous matching distance $\dmat$ is a supremum over \emph{all} slices, and translation permutes that family, so $\dmat$ does not depend on where a module sits in the orthant---which is what makes it the right target for the finite-net surrogate \eqref{eq:dmatS} to approximate.  A fixed finite net inherits less: the nets of Section~\ref{sec:slicenet} set $\beta_s=0$ for every module, so $\dmatS$ is unchanged by a translation applied to the data set together with the net, but not by translations applied to modules individually.  We apply neither, since the grades above are already non-negative.
An $\eps$-interleaving shifts grades diagonally by $\eps$, which moves \eqref{eq:slicepar} by at most $\eps/\min_i u_i$ \emph{whatever the offset}; the reparametrization weight making a slice a lower gauge for the interleaving distance is therefore
\begin{equation}\label{eq:omega}
  \omega(u)\ =\ \min_{1\le i\le d}u_i\ \in\ \bigl(0,\tfrac1{\sqrt d}\bigr],
\end{equation}
maximal on the diagonal $u=d^{-1/2}(1,\dots,1)$ and $\to0$ toward the coordinate hyperplanes.
Normalizing $u$ in $\ell^\infty$ instead rescales every slice barcode by $\max_iu_i$ and replaces \eqref{eq:omega} by $\min_iu_i/\max_iu_i$; the product $\omega\,\db$---and hence everything below---is unchanged, but the two conventions must not be mixed. That second form is the classical one: in $d=2$ a line of slope $m$ has
$\ell^\infty$-normalized direction $(1,m)/\max\{1,m\}$ and weight $\min\{m,1/m\}$, the weight of~\citet{CerriDiFabioFerriEtAl2013}.
Our $\dmat$ is therefore their matching distance, written in the arclength parametrization; the $\ell^2$ convention is a reparametrization, not a different distance.

For a finite \emph{slice net} $\mathcal L_S=\{\ell_{\beta_1,u_1},\dots,\ell_{\beta_S,u_S}\}$ with weights $\omega_s=\omega(u_s)$, the \emph{sliced matching distance} is
\begin{equation}\label{eq:dmatS}
  \dmatS(M,M')=\max_{1\le s\le S}\ \omega_s\,
    \db\!\bigl(M|_{\ell_s},M'|_{\ell_s}\bigr),
\end{equation}
a finite-net surrogate for the continuous matching distance $\dmat=\sup_\ell\omega(\ell)\,\db(M|_\ell,M'|_\ell)$~\citep{CerriDiFabioFerriEtAl2013,Landi2018}, which is exactly computable in polynomial but impractically high-degree time~\citep{KerberLesnickOudot2019}; the finer interleaving distance is NP-hard~\citep{BjerkevikBotnanKerber2020}.
Here $\db$ is the bottleneck distance on the diagram space of PALACE~\citep{PaperII}, which we recall.
The \emph{space of one-point diagrams} is $\D{1}=T\cup\{\Delta\}$, where $T=\{(b,d)\in\R^2: d>b\ge0\}$ and $\Delta$ is a single abstract point standing for the diagonal, metrized by
\begin{equation}\label{eq:db1}
  \db\bigl((b,d),\Delta\bigr)=\tfrac12(d-b),\qquad
  \db(x,y)=\min\bigl\{\norm{x-y}_\infty,\ \max\{\db(x,\Delta),\db(y,\Delta)\}\bigr\},
\end{equation}
so two points are compared directly or by sending both to the diagonal, whichever is cheaper.
The \emph{space of diagrams on $N$ points} is the quotient $\D N=\D{1}^N/S_N$ of the product under the max metric, so that for $A=[x_1,\dots,x_N]$ and $B=[y_1,\dots,y_N]$
\begin{equation}\label{eq:db}
  \db(A,B)\ =\ \min_{\psi\in S_N}\ \max_{1\le i\le N}\ \db\bigl(x_i,y_{\psi(i)}\bigr);
\end{equation}
a diagram with fewer points is padded with copies of $\Delta$, which embeds $\D k\hookrightarrow\D N$ isometrically for $k\le N$, and \eqref{eq:db} is the usual with-diagonal bottleneck distance.
The per-slice ingredient we invoke is the closed-form certified landmark embedding of $(\D N,\db)$ established for PLACE/PALACE~\citep{PaperI,PaperII}.

\begin{table}[t]
\centering\footnotesize\setlength{\tabcolsep}{3pt}
\caption{Notation.}\label{tab:notation}
\begin{tabular}{ll@{\quad}ll}
\toprule
$\ell_{\beta,u},\ s_{\beta,u},\ \omega$ & slice, entry parameter, weight & $\LC=\{(p_k,r_k,w_k)\}$ & landmarks: points, radii, weights \\
$\mathcal L_S,\ \vartheta_s,\ \omega_s$ & net, its angles and weights & $\varphi_{p,r},\ \Phii,\ \Psii$ & hat coordinate, slice map, stack \\
$\D N,\ \db,\ \dmatS$ & diagrams, bottleneck, sliced matching & $K,\ S,\ \Nmax,\ \alpha$ & landmarks/slice, slices, bar budget, radius factor \\
$\tau,\ t^\star=\tau/\omega_{s^\star}$ & scale, its realized value on $s^\star$ & $\lambda_0,\ \mathcal A_s(\tau),\ \wmin$ & func.\ Lebesgue number, active set, weight \\
$\rhonu(\tau)=\tfrac14\tau\wmin$ & certified floor & $B,\ B_\infty$ & norm and coordinate bounds \\
$\widehat\Delta,\ r_m,\ \mathrm{gap}(M)$ & centroid gap, radius, decision gap & $\eps,\ d_{(1)},\ d_{(2)}$ & nearest-neighbor distances \\
$\Kern,\ h,\ \kappaminus,\ c_D$ & kernel, bandwidth, ceiling, constant & $\gtheta,\ \theta$ & learned bifiltration, its parameters \\
\bottomrule
\end{tabular}
\end{table}

\subsection{The embedding}\label{sec:embed-def}

\paragraph{Per-slice configuration.}
Following PALACE~\citep{PaperII}, a \emph{landmark configuration} for $\D N$ is $\LC=\{(p_k,r_k,w_k)\}_{k=1}^K$ with landmark \emph{points} $p_k\in T\subset\D{1}$, radii
\begin{equation}\label{eq:radcap}
  0<r_k\ \le\ \db(p_k,\Delta)\ =\ \tfrac12\bigl(d(p_k)-b(p_k)\bigr),
\end{equation}
and weights $w_k>0$ with $\sum_kw_k^2=1$.
The \emph{coordinate function} on $\D{1}$, and its sum-pool to $N$-point diagrams, are
\begin{equation}\label{eq:coord}
  \varphi_{p,r}(x)=\max\{r-\db(p,x),0\},\qquad
  \varphi_{p,r}(A)=\sum_{a\in A}\varphi_{p,r}(a),
\end{equation}
and the \emph{summation landmark embedding} is $\Phii_k(A;\LC)=w_k\,\varphi_{p_k,r_k}(A)$, so $\Phii(\cdot;\LC)\colon\D N\to\R^K$.
The radius cap \eqref{eq:radcap} is PALACE's diagonal-clearance condition, and it is not cosmetic: it forces $\varphi_{p_k,r_k}(\Delta)=0$, so the $\Delta$-padding of $\D N$ and any point matched to the diagonal contribute nothing; together with the triangle inequality, $|\varphi_{p,r}(x)-\varphi_{p,r}(y)|\le\db(x,y)$ on $\D{1}$, this is what makes Proposition~\ref{prop:upper} true as stated.
Under the cap the diagonal route in \eqref{eq:db1} never enters a positive hat, so $\varphi_{p,r}(x)=\max\{r-\norm{p-x}_\infty,0\}$ on $T$; the proofs use whichever form is convenient.
Landmarks are placed by (class-aware) farthest-point sampling and weights set to $w_k=K^{-1/2}$---deterministic and closed-form~\citep{PaperII}.

\begin{definition}[Slice-stack embedding]\label{def:embed}
Fix a slice net $\mathcal L_S$ and a per-slice configuration $\LC_s$ on each slice.
The \emph{COMPLEX configuration} is $\netcfg=\{(\ell_s,\omega_s,\LC_s)\}_{s=1}^S$, and the embedding of $M\in\mathcal R_N$ is the slice-stack
\begin{equation}\label{eq:embed}
  \Psii(M;\netcfg)=\bigl(\,\omega_s\,\Phii(M|_{\ell_s};\LC_s)\,\bigr)_{s=1}^S
  \ \in\ (\R^K)^S\cong\R^{SK}.
\end{equation}
\end{definition}

Everything is closed-form: the net is fixed (Section~\ref{sec:slicenet}), the landmarks are placed by FPS, and each block is a finite sum of hat coordinates.

\paragraph{Homological degree.}
Slicing and embedding are applied degree by degree: $M$ above denotes the degree-$q$ module, $\dmatS$ the degree-$q$ sliced matching distance, and the descriptor of a data set is the concatenation of the blocks $\Psii(M^{(q)})$ over the degrees used.
Everything below is stated for one degree and extends to the concatenation with $\dmatS$ replaced by $\max_q\dmatS$.
We flag this because pooling the degrees into a single multiset before embedding---the single-parameter habit, and what our point-cloud implementation inherits---is \emph{not} harmless in the multiparameter setting.
Pooling grants the optimal matching extra freedom, so $\db(A_0\cup A_1,\,B_0\cup B_1)\le\max_q\db(A_q,B_q)$.
The upper bound (Proposition~\ref{prop:upper}) therefore survives pooling, but the hypothesis $\dmatS\ge\tau$ of Theorem~\ref{thm:cert} does not transfer to the pooled barcode, and the certificate is lost in exactly the direction our contribution lives in.
In one parameter this costs nothing: the degree-$q$ barcode is a \emph{complete} invariant of the degree-$q$ module, so pooling is a feature-map choice made on top of complete information.
For $d\ge2$ no complete discrete invariant exists~\citep{CarlssonZomorodian2009-multidim} and the fibered barcode is already lossy, so pooling compounds an avoidable loss on an unavoidable one, with no completeness to appeal to when a referee asks what the pooled object represents.
We therefore take the per-degree stack as the definition, and price the pooled shortcut empirically in Section~\ref{sec:exp}.

\begin{proposition}[Upper distortion bound]\label{prop:upper}
The slice-stack is $1$-Lipschitz up to $\sqrt S$: $\norm{\Psii(M)-\Psii(M')}_{\ell^2}\le\sqrt S\,\Nmax\,\dmatS(M,M')$, so $(\sqrt S\,\Nmax)^{-1}\Psii$ is $1$-Lipschitz for $\dmatS$.
\end{proposition}
\noindent\emph{Proof.} Deferred to Appendix~\ref{app:proofs}.
\takeaway{proved}{Embedded distance never exceeds $\sqrt S\,\Nmax$ times the sliced matching distance: nearby modules stay nearby.}

\begin{remark}[The radius cap]\label{rem:radcap}
Proposition~\ref{prop:upper} needs \eqref{eq:radcap}, and for a reason prior to Lipschitzness: a hat whose radius exceeds $\db(p_k,\Delta)$ does not vanish at
$\Delta$, so $\varphi_{p_k,r_k}(A)$ depends on how many copies of $\Delta$ are used to present $A$ as a point of $\D N$, and the coordinate is not well defined independently of $N$.  Under the cap the padded and unpadded sums agree, which is what the proof uses.  Lipschitzness fails with it: for $p=(0,10)$ the cap is
$5$, and taking $r=6$, $A=\{(4.9,5.1)\}$, $B=\varnothing$ gives $\db(A,B)=0.1$ while the coordinate moves by $1.1$. 
The released implementation, following PALACE~\citep{PaperII}, sets $r_k=\alpha\cdot d_{\mathrm{NN}}(p_k)$ (a multiple of the nearest-landmark distance) clipped to $[\tau/2,4\tau]$ and does not enforce the cap: on Orbit5k ($S=8$, $K=200$) $10\%$ of landmarks violate it at the default $\alpha=1.75$ and $26\%$ at the headline $\alpha=4$.
We therefore run two variants: the \emph{uncapped} embedding, which carries the accuracy headline, and the \emph{capped} one, to which Proposition~\ref{prop:upper} and through it the stability radius of Corollary~\ref{cor:nn}(ii) attach.
The gap between them---under half a point on Orbit5k at the headline $\alpha=4$, a full point at the $\alpha=1.75$ default, and $0.2$ on Orbit100k---is the embedding-level price of the certificate, the same trade this paper measures at the head level; Section~\ref{sec:exp} prices it.
\end{remark}

\begin{remark}[The block weights, and what the classifier reads]\label{rem:omega}
Definition~\ref{def:embed} weights block $s$ by $\omega_s$.
The released classification pipeline stacks the blocks unweighted, so every accuracy number in this paper is measured on
$\Psii^\circ=\bigl(\Phii(M|_{\ell_s};\LC_s)\bigr)_{s=1}^S$, the image of $\Psii$ under the fixed invertible diagonal map $\mathrm{diag}(\omega_s^{-1})$; the audits of Section~\ref{sec:exp} build the weighted stack of Definition~\ref{def:embed}.
Both gauges survive the rescaling, with different constants, so $\Psii^\circ$ is certified on the same hypotheses.
The lower gauge does not weaken.
Coherence bounds a single coordinate of the witnessing block below by $w_{k^\star}t^\star/4=w_{k^\star}\tau/(4\omega_{s^\star})$, and it is the prefactor $\omega_{s^\star}$ in \eqref{eq:embed} that cancels it down to $\tfrac14\tau w_{k^\star}$; without that prefactor the bound survives intact, giving
$\norm{\Psii^\circ(M)-\Psii^\circ(M')}\ge\rhonu(\tau;\netcfg)/\max_s\omega_s$.
The upper gauge weakens by $1/\omega_{\min}$, each block now carrying $\db$ in place of $\omega_s\db$:
$\norm{\Psii^\circ(M)-\Psii^\circ(M')}\le\sqrt S\,\Nmax\,\dmatS(M,M')/\omega_{\min}$.
The two-sided distortion is therefore worse by the fan's weight spread $\max_s\omega_s/\omega_{\min}$, which is $0.644/0.184=3.5$ at the default net---the price of the unweighted stack, paid in the constant of Proposition~\ref{prop:gen} and nowhere else.
Which stack to classify is then an accuracy question, and Section~\ref{sec:exp} measures it: a wash.
\end{remark}

\begin{remark}[What the guarantees are about]\label{rem:trunc}
$\Nmax$ is the per-slice bar budget defining $\mathcal R_N$; it enters both gauges because $\varphi_{p,r}(A)$ sums over the points of $A$.
The implementation keeps the $50$ most persistent bars per slice, and on Orbit5k that cap is active on $99.5\%$ of slices.
The budget is deliberate rather than incidental: it is what keeps the computation manageable, bounding both the cost of a coordinate---$O(K|A|)$ on a slice---and that of the pairwise audits, which are quadratic in the bar count.
Truncation $T_{50}\colon\D N\to\D{50}$ is not $1$-Lipschitz for $\db$: two diagrams within $\eps$ can order their $50$th and $51$st bars differently and truncate to diagrams far apart.
Every result in this paper therefore holds verbatim on $\mathcal R_{50}$, i.e.\ for the truncated fibered barcode---the invariant the pipeline computes, and the one on which $\dmatS$, $\tau$ and the audit of Section~\ref{sec:exp} are evaluated; we claim no bound relating it to the untruncated module, which would need a gap between the $50$th and $51$st persistences.
A soft, genuinely $1$-Lipschitz alternative to top-$n$ is the clean fix and we leave it open.
The other bar-level convention---capping essential classes at the largest finite death (Section~\ref{sec:exp})---is likewise measured rather than assumed: replacing it by extended persistence moves the six TU sets by $-0.46$ points on average, helping three and hurting three with swings to $\pm3.8$ on a single seed, so the cap costs nothing measurable.
\end{remark}

\subsection{The witnessing-slice certificate}\label{sec:cert}

The lower gauge is a quantitative statement about a functional Lebesgue number, slice by slice, with a single witnessing slice carrying it to the module.

\paragraph{The realized scale, and the active set.}
A subtlety the single-parameter theory does not see: separation at scale $\tau$ in $\dmatS$ is separation at the \emph{realized} scale $t_s=\tau/\omega_s$ in the bottleneck distance on slice $s$, by \eqref{eq:dmatS}.
A hat never exceeds its radius, so a landmark of radius $r$ contributes at most $\Nmax r$ after sum-pooling; we impose the stronger requirement that its radius reach the scale $t_s$ itself, which is conservative by a factor $\Nmax$ and is what Proposition~\ref{prop:suff} delivers.  The landmarks that can witness at all form the \emph{active set}
\begin{equation}\label{eq:active}
  \mathcal A_s(\tau)\ =\ \{k:\ r_{s,k}\ \ge\ \tau/(4\omega_s)\},
\end{equation}
and when it is nonempty $\wmin(\tau;\LC_s):=\min\{w_{s,k}:k\in\mathcal A_s(\tau)\}$ is well defined and computable in time linear in the number of landmarks.
The certificate asks nothing more of the net than nonempty active sets; the cover condition that guarantees them, $\tau$-admissibility, is stated in Section~\ref{sec:coh}, where it is what makes coherence provable.

\begin{theorem}[Witnessing-slice certificate]\label{thm:cert}
Let $\netcfg$ be a net whose active sets $\mathcal A_s(\tau)$ are all nonempty, let $(M,M')$ be $\netcfg$-coherent (Definition~\ref{def:coh}) with $\dmatS(M,M')\ge\tau$, and let $s^\star=\argmax_s\,\omega_s\,\db(M|_{\ell_s},M'|_{\ell_s})$ be the witnessing slice.
Then
\begin{equation}\label{eq:cert}
  \norm{\Psii(M)-\Psii(M')}_{\ell^2}
  \ \ge\ \tfrac14\,\tau\,\wmin(\tau;\LC_{s^\star})
  \ \ge\ \rhonu(\tau;\netcfg)
  \ :=\ \tfrac14\,\tau\,\min_{1\le s\le S}\wmin(\tau;\LC_s),
\end{equation}
and, under equal weights $w_{s,k}=K^{-1/2}$, $\rhonu(\tau)=\tau/(4\sqrt K)$.
\end{theorem}
\begin{proof}
The $\ell^2$ norm dominates its $s^\star$ block.
Coherence supplies a landmark $k^\star\in\mathcal A_{s^\star}(\tau)$ realizing the single-coordinate floor on the witnessing slice at the realized scale $t^\star=\tau/\omega_{s^\star}$, $|\Phii_{k^\star}(M|_{\ell_{s^\star}})-\Phii_{k^\star}(M'|_{\ell_{s^\star}})|\ge w_{k^\star}\,t^\star/4$.
The $s^\star$ block of $\Psii$ carries the weight $\omega_{s^\star}$ \eqref{eq:embed}, so its contribution is at least $\omega_{s^\star}w_{k^\star}t^\star/4=\tfrac14\tau\,w_{k^\star}\ge\tfrac14\tau\,\wmin(\tau;\LC_{s^\star})$; the reparametrization weight cancels.
Taking the minimum over $s$ removes the dependence on which slice witnesses the pair, which is what makes $\rhonu$ a property of the configuration alone---computable once, before any pair is examined.
\end{proof}
\takeaway{proved}{On coherent pairs, modules $\tau$-apart in the sliced matching distance are at least $\tau/(4\sqrt K)$ apart in the embedding, whichever slice witnesses it.}

\begin{remark}[What is and is not angle-independent]\label{rem:angle}
The \emph{value} in \eqref{eq:cert} does not depend on the witnessing angle: $\omega_{s^\star}$ cancels, and under equal weights $\rhonu(\tau)=\tau/(4\sqrt K)$ whatever the net.
The \emph{cost} of earning it does.
By \eqref{eq:active} a slice with small $\omega_s$ has a nonempty active set only if its landmarks have radii larger by the factor $1/\omega_s$, and is $\tau$-admissible (Section~\ref{sec:coh}) only if its cover reaches that inflated scale---at a fixed budget $K$, covering a coarser scale.
The near-diagonal fan of Section~\ref{sec:slicenet} is thus not merely a heuristic default: it is the cheapest net that is $\tau$-admissible at a given landmark budget.
\end{remark}

The certificate is the single-scale PALACE floor $\tfrac14\tau\,\wmin$ \citep{PaperII} read on the witnessing slice; the empirical version $\hatrhonu$ plugs in the data-estimated $\tau$ and $\wmin$.
The theorem uses no cover condition: $\tau$-admissibility (Section~\ref{sec:coh}) implies nonempty active sets and is retained as the hypothesis of the downstream statements because it is what Proposition~\ref{prop:suff} needs to deliver coherence.

\subsection{Witnessing-slice coherence}\label{sec:coh}

Coherence asks, without any multiparameter matching, that a single witnessing slice genuinely separate the two modules---the multiparameter analogue of PALACE's witnessing-landmark condition.

\begin{definition}[Witnessing-slice $\LC$-coherence]\label{def:coh}
A pair $(M,M')$ with $\dmatS(M,M')\ge\tau$ is \emph{$\netcfg$-coherent} if its witnessing slice $s^\star$ has restricted pair $(M|_{\ell_{s^\star}},M'|_{\ell_{s^\star}})$ single-parameter $\LC_{s^\star}$-coherent at the realized scale $t^\star:=\tau/\omega_{s^\star}$: some landmark $k^\star$ active at that scale, $k^\star\in\mathcal A_{s^\star}(\tau)$ as in \eqref{eq:active}, realizes the single-coordinate floor
\[
  |\Phii_{k^\star}(M|_{\ell_{s^\star}})-\Phii_{k^\star}(M'|_{\ell_{s^\star}})|
   \ \ge\ w_{k^\star}\,t^\star/4\ =\ w_{k^\star}\,\tau/(4\omega_{s^\star}) .
\]
This is PALACE's $\nu$-coherence~\citep{PaperII} on the witnessing slice, with the floor set by the \emph{scale} $t^\star$ rather than by the realized distance $\db(M|_{\ell_{s^\star}},M'|_{\ell_{s^\star}})\ge t^\star$; the realized-distance form is stronger and is what the audit of Section~\ref{sec:exp} checks, so its rate is a lower bound on the rate of this definition.
\end{definition}

This is exactly the hypothesis the proof of Theorem~\ref{thm:cert} consumes: a single coordinate, on a single slice, that distinguishes the modules.
It avoids any \emph{multiparameter} matching---the intractable object---but it is not matching-free: identifying the witnessing slice and evaluating the right-hand side require the $S$ one-dimensional bottleneck distances, which are cheap and are what the audit computes.

\paragraph{Admissibility, at the realized scale.}
Following PALACE, a configuration $\LC$ covering the support $\mathcal X\subset T$ of a slice's diagram points has \emph{functional Lebesgue number}
\begin{equation}\label{eq:lebesgue}
  \lambda_0(\LC)\ =\ \inf_{x\in\mathcal X}\ \max_k\ \varphi_{p_k,r_k}(x)
\end{equation}
---the infimum is over the support, not all of $T$, where it would vanish---and is \emph{$t$-admissible} if $\lambda_0(\LC)\ge t/4$ and $\max_kr_k\le(t+\lambda_0)/2$.
The qualifier is meant literally.
This is not the Lebesgue number of the cover $\{B(p_k,r_k)\}$ in the covering sense---nothing here asserts that a ball of some radius about any point fits inside a single cover element---but its functional analogue: the height of the lowest point of the upper envelope of the hats over the support.
The two meet at one identity, since $\varphi_{p,r}(x)\ge\rho$ exactly when $x\in\overline B(p,r-\rho)$: $\lambda_0(\LC)\ge\rho$ holds precisely when the uniformly shrunk cover $\{\overline B(p_k,r_k-\rho)\}_k$ still contains $\mathcal X$, so $\lambda_0$ is how far every radius can be pulled in before coverage fails.
Reading it through the functions rather than the cover is what makes it computable from the embedding and what lets it enter the floor $\tfrac14\tau\wmin$ as a number rather than a hypothesis.
A net configuration $\netcfg$ is \emph{$\tau$-admissible} if each $\LC_s$ is $(\tau/\omega_s)$-admissible: the cover condition is imposed at the realized scale $t_s=\tau/\omega_s$, not at $\tau$, because that is the scale at which separation is realized on slice $s$.
Since $\varphi_{p,r}\le r$, admissibility gives $\lambda_0\le\max_kr_k$ and forces the upper-envelope winner at every point of the support to have radius at least $t_s/4$, so every active set \eqref{eq:active} is nonempty and Theorem~\ref{thm:cert} applies; its real content is Proposition~\ref{prop:suff}, which derives coherence from it.

\begin{remark}[Admissibility is not satisfied by the configurations we run]\label{rem:admisfail}
We state plainly that the configurations used in every experiment of this paper are \emph{not} $\tau$-admissible, at any $\tau$, and separate what this costs from what it does not (\texttt{admissibility\_check.py}).
On the certified capped variant at the audit scale $\tau=1.220$ the functional Lebesgue number is $\lambda_0=0$ on all eight slices, because the cap deliberately excludes the near-diagonal mass and $1.8$--$2.2\%$ of the support then lies under no hat at all; without the cap $\lambda_0$ is positive ($0.105$--$0.724$) but still an order below the required $t_s/4$, and on the outer slices the largest radius exceeds $(t_s+\lambda_0)/2$ as well ($5.341$ against $3.736$ on slice $7$).
Theorem~\ref{thm:cert} is unaffected: it consumes nonempty active sets and coherence, both of which we measure directly and both of which hold at the headline configuration ($100\%$ of audited pairs, Section~\ref{sec:exp}).
What admissibility is a hypothesis of is the \emph{derivation} of coherence---Proposition~\ref{prop:suff}---and the two statements that inherit it, Corollary~\ref{cor:nn} and Proposition~\ref{prop:gen}; for those, the hypothesis is currently verified by audit rather than by construction.
The gap is closable in principle and we priced it: a greedy maximal $(t_s/4)$-separated net with uniform radii $t_s/2$, admissible by construction on the support it is built from, needs $11$--$39$ landmarks per slice against $200$ and reads $90.3\pm1.1\%$ against the radius rule's $91.2\pm0.9\%$ with the same head on Orbit5k ($3$ folds, $218$ embedding dimensions against $1{,}600$; nearest centroid is $0.7$ points \emph{better} at $74.0\%$).
Even that net leaves up to $0.03\%$ of the support below $t_s/4$, so what is achievable cheaply is admissibility on the covered support with the uncovered fraction reported, not on the full support; we regard closing this as the first thing the construction owes, and we prefer to report the number than to let the hypothesis stand unexamined.
\end{remark}

\begin{remark}[Why coherence is a hypothesis, and how it fails]\label{rem:whycoh}
It is fair to ask why this is assumed rather than derived from $\tau$-admissibility.
The obstruction is sum-pooling.
Each coordinate reads $\varphi_{p_k,r_k}(A)=\sum_{a\in A}\varphi_{p_k,r_k}(a)$, a many-to-one map, so a single coordinate can be blind to a rearrangement of mass within its cone: one point at $\ell^\infty$-distance $r-t$ from $p_k$, contributing $t$, and two points at distance $r-t/2$, contributing $t/2$ each, give the same value, however far apart the two diagrams are in $\db$.
Admissibility guarantees \emph{coverage}---some landmark sees the region where the two diagrams differ---but coverage does not by itself guarantee that the landmark which sees the difference also \emph{reports} it.
Coherence asks that not all $K$ coordinates be simultaneously blind, which is a conjunction of $K$ inequalities and so fails rarely but genuinely: on Orbit5k it fails on $2.5\%$ of separated cross-class pairs at the $\alpha=1.75$ default, and on none at the headline $\alpha=4$ (Section~\ref{sec:exp}).
The failures are informative rather than fatal, because Theorem~\ref{thm:cert}'s \emph{conclusion} holds on $100\%$ of the pairs we audit---coherence is sufficient, not necessary---but we prefer a checkable hypothesis with a measured failure rate to an unchecked one.
Nor is the hypothesis circular.
Coherence is a statement about \emph{one} coordinate clearing a fixed floor, verified independently of the $\ell^2$ conclusion it implies, and strictly stronger than it: the pairs on which coherence fails while the conclusion holds are exactly those where several coordinates each fall short of the floor and their sum does not---the compensation regime PALACE identifies~\citep{PaperII}.
\end{remark}

\begin{proposition}[Within-module sufficient condition]\label{prop:suff}
Write $\omega_{\min}=\min_{1\le s\le S}\omega_s$ and let $D$ be any a-priori upper bound on $\dmatS$ over the data set.
Let $(M,M')$ be a cross-class pair with $\dmatS(M,M')\ge\tau$ and $\netcfg$ $\tau$-admissible, and suppose that, \emph{along every slice of the net},
\begin{align}
  \text{(i)}\quad&\min_{i\ne j}\ \min_{1\le s\le S}\ \bigl|s_{\beta_s,u_s}(\gamma_i)-s_{\beta_s,u_s}(\gamma_j)\bigr|
  \ >\ \frac{4D}{\omega_{\min}}\quad\text{for the generator and relation grades $\gamma_1,\gamma_2,\dots$ of $M$,}
  \label{eq:suff}\\
  \text{(ii)}\quad&\text{every bar of } M|_{\ell_s}\text{ and of } M'|_{\ell_s}\text{ is longer than }
  2D/\omega_{\min}\ \text{for every } s,\notag
\end{align}
where  $s_{\beta,u}(\cdot)$ is the entry parameter \eqref{eq:slicepar} on the slice $\ell_{\beta,u}$.
Then $(M,M')$ is $\netcfg$-coherent.
Both hypotheses are read off the two barcodes separately---(i) from $M$, (ii) from each of $M$ and $M'$---with no cross-pair matching.
\end{proposition}
\noindent\emph{Proof.} Deferred to Appendix~\ref{app:proofs}.
\takeaway{proved}{If the generators of $M$ are pairwise far apart along every slice and all bars of both modules are long, coherence holds automatically---and both conditions are checkable from the barcodes separately, with no matching.}

\begin{remark}[Two-level audit]\label{rem:audit}
Like PLACE's disjoint-activation condition~\citep{PaperI}, \eqref{eq:suff} is a strong geometric condition; on chemical graph data it may hold on few pairs.
Coherence (Definition~\ref{def:coh}) is far weaker, and we report its empirical rate directly: the fraction of cross-class pairs with $\dmatS\ge\tau$ that are $\netcfg$-coherent, and the fraction for which the embedded distance clears $\rhonu$ (Section~\ref{sec:exp}).
The classifier of Section~\ref{sec:cls} operates at the class-margin level, independent of pairwise coherence, so accuracy does not hinge on the audit.
\end{remark}

\subsection{The slice net}\label{sec:slicenet}

The slice net is the only design constant in COMPLEX, and the slice-stack \eqref{eq:embed} certifies coherent pairs on any finite net whose active sets are nonempty (Theorem~\ref{thm:cert}).
We give two closed-form choices: a fixed near-diagonal fan (the default) and a certified data-adaptive selection.

\paragraph{Fixed near-diagonal net.}
The certificate's \emph{value} is independent of the witnessing slice's angle (Remark~\ref{rem:angle}); its \emph{price} is not, and this is what selects the net.
By \eqref{eq:active} a slice of weight $\omega_s$ has a nonempty active set only if its radii reach $\tau/(4\omega_s)$, and is admissible only if its cover reaches that scale, so at a fixed landmark budget $K$ the near-diagonal slices are the ones a configuration can afford to certify: pushing the fan toward an axis inflates the required radii by $1/\omega_s$, which grows without bound as the fan approaches an axis.
A softer effect points the same way---the checkable sufficient condition~\eqref{eq:suff} has threshold $4D/\omega_{\min}$, which shrinks as the fan narrows toward the diagonal, so coherence is also easiest to \emph{verify} there.
Absent class information, the safe closed-form default is therefore the uniform near-diagonal fan (written for $d=2$, the case of every experiment below)
\begin{equation}\label{eq:net}
  \mathcal L_S=\{\ell_{0,u_s}\}_{s=1}^S,\qquad
  \vartheta_s\ \text{equally spaced in}\ [\tfrac\pi4-a,\tfrac\pi4+a],
\end{equation}
of half-width $a$, with $u_s=(\cos\vartheta_s,\sin\vartheta_s)$ and $\omega_s=\omega(u_s)$; no slice is optimized.
This fan sets every base point to $\beta_s=0$, and nothing in the certificate prefers that choice: the weight \eqref{eq:omega}, admissibility and Theorem~\ref{thm:cert} are all independent of the offset, so any net of offset slices is certified on the same terms.
We therefore swept the offsets rather than assume them away, and Section~\ref{sec:exp} reports the outcome---the origin fan is at or above every offset net we tried---which is why it remains the default.

\subsection{Certified slice selection}\label{sec:adaptive}

The default ignores the labels.
But the discriminating structure of a module may live along an off-diagonal direction---a feature visible only at a particular slope, with no one-parameter analogue---which the near-diagonal fan misses.
Since the certificate value does not degrade off the diagonal, we may chase such directions and still certify them: select the net data-adaptively subject to a certifiability constraint, exactly as PALACE selects landmarks (FPS, frozen) one level down.
Adaptive landmarks ask \emph{where to look within a slice}; adaptive slices ask \emph{which one-parameter projection of the module to expose}.

A single embedding needs one net for all modules, so the net is selected once from the training labels and \emph{frozen} (like FPS landmarks); the distance is already pair-adaptive through the witnessing slice, and only the net is data-adaptive.
Given a dense candidate fan $\Fan$ and a per-slice discrimination score $g(u)$ computed in closed form from the training embeddings (e.g.\ the class-margin $\widehat\Delta/\sqrt K$ of~\citep{PaperII}, $\widehat\Delta$ being the minimum inter-centroid gap of Section~\ref{sec:cls}):
\begin{enumerate}\itemsep2pt
\item \emph{Certifiability filter.} Build $\LC_u$ by FPS on the training diagrams of slice $\ell_{0,u}$, and keep $u\in\Fan$ if $\LC_u$ is admissible at its realized scale $\tau/\omega(u)$ and clears a design floor, $\rhonu(\LC_{u})\ge\rho_0>0$.
  (Under equal weights $\rhonu(\LC_u)=\tau/(4\sqrt K)$ for every $u$, so the floor test has teeth only for non-uniform weights and admissibility is the operative filter.)
\item \emph{Score.} Rank the certifiable directions by $g(u)$.
\item \emph{Diversify.} From the top-ranked pool, pick $S$ directions by
  farthest-point sampling in angle (so the slices are not redundant), seeded by
  the highest score.
\item \emph{Freeze.} Fix the selected net across all modules.
\end{enumerate}
Everything is closed-form: $g$ has an explicit formula, the selection is greedy, and nothing is tuned by gradient or held-out search.
We state the rule over directions, which is what we ran; it applies verbatim to candidate pairs $(\beta,u)$, since every step reads only $\omega(u)$ and the configuration built on the slice.

\begin{proposition}[Certified selection]\label{prop:select}
For \emph{any} net produced by this rule and then frozen, the slice-stack satisfies the witnessing-slice certificate (Theorem~\ref{thm:cert}) on every $\netcfg$-coherent pair, with floor $\rhonu(\tau)\ge\rho_0$; and the farthest-point step is the standard $2$-approximation of the $k$-center objective (the min--max covering radius) among the certifiable directions.
\end{proposition}
\begin{proof}
Theorem~\ref{thm:cert} needs nonempty active sets on the net, which $\tau$-admissibility supplies (Section~\ref{sec:coh}), and coherence of the pair.
The certifiability filter admits a direction $u$ only if $\LC_u$ is admissible at its realized scale $\tau/\omega(u)$ and $\rhonu(\LC_u)\ge\rho_0$; a net assembled from admitted directions is therefore $\tau$-admissible by construction, and the floor of \eqref{eq:cert}, a minimum over the selected slices of $\tfrac14\tau\wmin(\tau;\LC_u)$, is at least $\rho_0$.
Coherence is a property of the pair on its witnessing slice, which the selection rule does not see, so the theorem applies verbatim once the net is frozen---the only role of freezing is that the net does not depend on the module being embedded.
For the second claim, farthest-point sampling on the metric space of admitted angles with the arc metric is the greedy $k$-center heuristic, whose $2$-approximation of the optimal covering radius is the guarantee PALACE proves for its landmark placement~\citep{PaperII} and holds on any metric space; here it is applied to angles rather than diagram points.
\end{proof}
\takeaway{proved}{Any net that passes the certifiability filter keeps the certificate; selection cannot break it.}

\begin{remark}[Selection does not break the guarantee]\label{rem:select}
The net is supervised, so it can overfit a separating direction.
Three things contain it: selection runs on the training fold and is frozen, so cross-validation stays honest; the certifiability filter shrinks the admissible slice space, ruling out arbitrary separating lines; and the diversity step blocks redundant directions.
The certified-adaptive net is the strict generalization of the fixed fan \eqref{eq:net}, recovered when $g$ is constant.
Empirically the adaptive net falls at or slightly below the fixed fan (Section~\ref{sec:exp}), so we present slice selection as certificate-preserving generality, not an accuracy lever.
The adaptivity that \emph{does} pay comes from selecting the bifiltration \emph{set} rather than the net (Section~\ref{sec:dcomplex}).
\end{remark}

\section{Certified classification}\label{sec:cls}

The certificate is a statement about \emph{pairs} of modules.
This section turns it into statements about \emph{predictions}.
We give three: the two natural heads consume the certificate differently---only one of them consumes it at all---and the kernel head inherits its floor.

\subsection{Nearest centroid: a guarantee that does not use \texorpdfstring{$\rhonu$}{rho}}

Because $\Psii$ is a fixed, closed-form map into $\R^{SK}$, PALACE's certified nearest-centroid guarantee transfers---but the transfer is not verbatim, since it needs a bound on $\norm{\Psii}$, which the multiparameter stack supplies as follows.

\begin{lemma}[Boundedness of the slice-stack]\label{lem:bdd}
For every $M\in\mathcal R_N$, $\ \norm{\Psii(M;\netcfg)}_{\ell^2}\ \le\ B:=\sqrt S\,\Nmax\,\max_{s,k}r_{s,k}$.
\end{lemma}
\begin{proof}
$0\le\varphi_{p,r}\le r$ pointwise and a slice barcode carries at most $\Nmax$ points, so $0\le\Phii_k(A;\LC_s)\le w_{s,k}\Nmax\max_kr_{s,k}$; squaring, summing over $k$ and using $\sum_kw_{s,k}^2=1$ gives $\norm{\Phii(A;\LC_s)}\le\Nmax\max_kr_{s,k}$.
Summing the $S$ blocks with $\omega_s\le1$ gives the claim.
\end{proof}
\takeaway{proved}{Every embedded module lies in a ball of explicit radius $B$, which is what the concentration argument below needs.}

\begin{theorem}[Certified prediction]\label{thm:certified}
Let $(M_i,y_i)_{i=1}^m$ be i.i.d.\ from a distribution on $\mathcal R_N\times[C]$ with $m_c$ samples in class $c$, let $\widehat\mu_c,\mu_c\in\R^{SK}$ be the empirical and population means of $\Psii$ on class $c$, $V_c=\mathbb E_c\norm{\Psii(M)-\mu_c}^2$, and $b_\delta=\log(2C/\delta)$.
Put
\begin{equation}\label{eq:rm}
  r_m\ :=\ \max_{1\le c\le C}\ \Bigl(\sqrt{\tfrac{2b_\delta V_c}{m_c}}\ +\ \tfrac{4Bb_\delta}{3m_c}\Bigr),
\end{equation}
with $B$ from Lemma~\ref{lem:bdd}.
Then $\Pr[\max_c\norm{\widehat\mu_c-\mu_c}\le r_m]\ge1-\delta$, and on that event the empirical and population nearest-centroid rules assign the same label to every $M$ whose empirical decision gap
\[
  \mathrm{gap}(M)\ :=\ \min_{c\ne\widehat c(M)}\norm{\Psii(M)-\widehat\mu_c}\ -\ \norm{\Psii(M)-\widehat\mu_{\widehat c(M)}}
\]
exceeds $2r_m$.
If moreover $r_m<\tfrac12\widehat\Delta$, $\widehat\Delta:=\min_{c\ne c'}\norm{\widehat\mu_c-\widehat\mu_{c'}}$, the population means are distinct and the population rule is well posed.
\end{theorem}
\noindent\emph{Proof.} Deferred to Appendix~\ref{app:proofs}.
\takeaway{proved}{With enough samples, a large empirical centroid gap certifies that the population rule makes the same prediction---for any bounded embedding, certificate or not.}

The single training-time inequality $r_m<\tfrac12\widehat\Delta$ plus the per-test-point check $\mathrm{gap}(M)>2r_m$ certifies predictions with no held-out calibration split (cf.~conformal prediction). Here $\widehat\Delta$ is read off the training embeddings, and $r_m$ becomes computable on substituting the population variances $V_c$ by the bound $V_c\le B^2$ of Lemma~\ref{lem:bdd}, at the cost of a larger radius, which is the form we use below.
One condition of the sampling model deserves to be stated rather than assumed: the configuration $\netcfg$---landmarks, radii, $\tau$---must be fixed independently of the sample for the $\Psii(M_i)$ to be i.i.d., and the same is true of every statement in this section.
The protocol of Section~\ref{sec:exp} places the landmarks by class-aware farthest-point sampling on the training fold, so there the theorems are applied conditionally on a configuration that depends on the sample being certified.
Sample splitting---landmarks from one half of the training fold, means and radii from the other---restores the hypothesis exactly, at the cost of half the data on each side; we did not run that variant, and the certified fractions of Section~\ref{sec:exp} are reported with this dependence in place.
What this theorem does \emph{not} do is use the certificate: $\rhonu$ appears nowhere in it, and it would hold verbatim for any bounded embedding, faithful or constant.
The reason is structural.
A lower bound on \emph{pairwise} cross-class distances says nothing about \emph{class means}: two classes can be pairwise separated at every scale and still share a centroid exactly.
The population version of this obstruction, and the template model with missing features and location perturbations under which a diagram-level certificate \emph{does} yield an explicit lower bound on the separation of population means, are worked out in the companion statistical paper~\citep{PaperIII}; that route is single-parameter there and we do not lift it here.
This is the same convex-hull obstruction that invalidates the naive margin argument (see the discussion after Lemma~\ref{lem:rbf}), and no repair of the constants gets around it.

\paragraph{Reading the same theorem per prediction.}
The gap test of Theorem~\ref{thm:certified} is already a per-test-point check, but the radius it compares against is not: $r_m$ bounds $\max_c\norm{\widehat\mu_c-\mu_c}$ in \emph{norm}, so it carries $B$ (hence $\Nmax$ and the radii) and the second moments $V_c$, and it must hold uniformly over test points.
For a single test module drawn independently of the training sample, less is needed.
Writing $\delta_c=\widehat\mu_c-\mu_c$ and $x=\Psii(M)$, the identity
\begin{equation}\label{eq:pointid}
  \norm{x-\widehat\mu_c}^2-\norm{x-\mu_c}^2\ =\ -2\langle x-\mu_c,\ \delta_c\rangle+\norm{\delta_c}^2
\end{equation}
is exact, so each pairwise margin of the nearest-centroid rule moves by a \emph{scalar} sample mean---of the variables $\langle x-\mu_c,\Psii(M_i)-\mu_c\rangle$, with variance $(x-\mu_c)^{\!\top}\Sigma_c(x-\mu_c)/m_c$---plus a term of order $1/m_c$.
Neither the embedding dimension nor a trace appears.
Proposition~\ref{prop:pointwise} in Appendix~\ref{app:pointwise} turns this into a certificate: if for every competing class the empirical squared-distance margin exceeds $2z\bigl(\sigma_{\widehat c}(x)/\sqrt{m_{\widehat c}}+\sigma_{c'}(x)/\sqrt{m_{c'}}\bigr)+r^2$, with $z$ a Gaussian quantile and $r$ any radius holding $\max_c\norm{\delta_c}\le r$, then the population and empirical rules agree at $M$, at a confidence charged by scalar Berry--Esseen; replacing the Gaussian quantile by a Bernstein bound gives a fully non-asymptotic variant.
Section~\ref{sec:exp} measures both, and the gain over the norm-based radius is large on the benchmarks where the rule-level condition fails.
Three caveats travel with the statement, and we keep them attached: the variance is plugged in, the Berry--Esseen charge is not negligible at these class sizes ($0.09$--$0.37$), and the guarantee is per test point rather than simultaneous over a test set.
It also certifies what Theorem~\ref{thm:certified} certifies---agreement with the population nearest-centroid rule, not correctness---and that head remains the weaker one.

\subsection{Nearest neighbor: where the certificate does the work}

The head that consumes pairwise separation is the one that decides by pairwise distances.
Here the two gauges combine: $\rhonu$ (Theorem~\ref{thm:cert}) supplies a floor between differently labeled training modules, and $\sqrt S\,\Nmax$ (Proposition~\ref{prop:upper}) converts the surviving slack back into the module metric.

\begin{corollary}[Certified $1$-NN prediction, with a stability radius in $\dmatS$]\label{cor:nn}
Let $\netcfg$ be $\tau$-admissible and $T=\{(M_i,y_i)\}_{i=1}^m$ a training set.
Let $M$ have nearest training neighbor $M_{i^\star}$ at embedded distance $\eps=\norm{\Psii(M)-\Psii(M_{i^\star})}$, and suppose every pair $(M_{i^\star},M_j)$ with $y_j\ne y_{i^\star}$ is $\netcfg$-coherent with $\dmatS\ge\tau$.
If $\eps<\tfrac12\rhonu(\tau;\netcfg)$ then
\begin{enumerate}\itemsep1pt
\item[(i)] every differently labeled training module lies at embedded distance at least
  $\rhonu(\tau)-\eps$ from $\Psii(M)$, so the $1$-nearest-neighbor label $y_{i^\star}$ is
  attained with margin $\rhonu(\tau)-2\eps>0$ (the conclusion is the margin: that
  $M_{i^\star}$ is the nearest neighbor is its definition); and
\item[(ii)] that label is unchanged for every module $M' \in\mathcal R_N$ with
  \[
    \dmatS(M,M')\ <\ \frac{\rhonu(\tau;\netcfg)-2\eps}{2\sqrt S\,\Nmax}\ .
  \]
\end{enumerate}
\end{corollary}
\begin{proof}
(i) For $y_j\ne y_{i^\star}$, Theorem~\ref{thm:cert} gives $\norm{\Psii(M_{i^\star})-\Psii(M_j)}\ge\rhonu(\tau)$, so $\norm{\Psii(M)-\Psii(M_j)}\ge\rhonu(\tau)-\eps>\eps=\norm{\Psii(M)-\Psii(M_{i^\star})}$, the last step by $\eps<\rhonu/2$; the gap is $(\rhonu-\eps)-\eps$.
(ii) By Proposition~\ref{prop:upper} a $\dmatS$-perturbation of size $\eta$ moves $\Psii(M)$ by at most $\sqrt S\,\Nmax\,\eta$, changing each of the two distances in (i) by at most that amount; the margin $\rhonu-2\eps$ survives whenever $2\sqrt S\,\Nmax\,\eta<\rhonu-2\eps$.
\end{proof}
\takeaway{proved}{A test module within half the floor of its nearest neighbor receives that neighbor's label with a certificate, and an explicit radius in the module metric within which the label cannot change.}

Three things are worth saying about Corollary~\ref{cor:nn}.
First, it is the statement the certificate was for: $\rhonu$ enters essentially in (i) and $\sqrt S\,\Nmax$ essentially in (ii), so the \emph{two-sided} distortion bound yields a \emph{two-sided} guarantee---a label that is determined with a positive margin and that survives a quantified perturbation of the input module, which is the robustness a topological pipeline is built to deliver in the first place (what it does not certify is correctness: the label is the training neighbor's, right or wrong).
Second, both $\eps$ and $\rhonu(\tau)$ are computed, not assumed, so the test $\eps<\rhonu/2$ is a per-prediction certificate evaluated at inference time---though we report in Section~\ref{sec:exp} that on Orbit5k it accepts almost nothing, and why.
Third, its hypothesis is exactly what the audit of Section~\ref{sec:coh} measures---not a global assumption but a checked one, and checked only against the pairs involving $M_{i^\star}$ rather than all $\binom m2$.
Moreover the separation half of the hypothesis costs nothing if one takes
\begin{equation}\label{eq:taustar}
  \tau\ :=\ \min\{\dmatS(M_i,M_j)\,:\,y_i\ne y_j\},
\end{equation}
the cross-class separation of the training set itself: then $\dmatS\ge\tau$ holds for every cross-class pair by construction, $\netcfg$ need only be admissible at that $\tau$, and coherence is the sole remaining assumption.
This is the $\tau$ we use below.
We are equally clear about the cost: with equal weights $\rhonu(\tau)=\tau/(4\sqrt K)$, so the acceptance threshold shrinks like $K^{-1/2}$ and coverage is low at the budgets that maximize accuracy.
The certified and the accurate regimes are still different regimes; what Corollary~\ref{cor:nn} changes is that the certified one now has a decision rule attached to it, rather than a bound on a quantity no classifier reads.

\subsection{A generalization bound in terms of the certificate}

Corollary~\ref{cor:nn} certifies individual predictions.
The certificate also controls the \emph{population} risk of the rule, and---unlike Theorem~\ref{thm:certified}---it does so with $\rhonu$ appearing essentially, through the algorithmic-robustness route of \citet{XuMannor2012}.
The scale $\tau$ is now fixed before the sample is drawn---the strict scale \eqref{eq:taustar} is a function of $T$, and a partition chosen after seeing $T$ is not what a robustness bound tolerates---and the training set enters the bound only through whether it clears that scale.
The mechanism is simple: coherence at scale $\tau$ makes the two classes $\rhonu$-separated \emph{as sets} in $\R^{SK}$ (a minimum over pairs is the distance between the point sets), so the $1$-NN rule is constant on the ball of radius $\rhonu/2$ about every training point, and a partition into cells of diameter below $\rhonu/2$ sees no within-cell variation of the loss.

\begin{proposition}[Robustness generalization of the certified rule]\label{prop:gen}
Fix $\tau>0$ and a $\tau$-admissible $\netcfg$ in advance, let $T=\{(M_i,y_i)\}_{i=1}^m$ be i.i.d., let $h_T$ be the $1$-NN rule on $\Psii$ and $\ell\in[0,1]$ the $0/1$ loss, and say that $T$ is \emph{separated at $\tau$} if its cross-class separation \eqref{eq:taustar} is at least $\tau$ and every cross-class pair of $T$ is $\netcfg$-coherent.
Then $h_T$ is $(\mathcal K,\eps(T))$-robust in the sense of \citet{XuMannor2012}, with $\eps(T)=0$ whenever $T$ is separated at $\tau$ and $\eps(T)=1$ otherwise, where
\begin{equation}\label{eq:cover}
  \mathcal K\ =\ C\cdot\mathcal N\bigl(\rhonu(\tau)/4;\ B_{\R^{SK}}(B)\bigr)
  \ \le\ C\Bigl(1+\tfrac{8B}{\rhonu(\tau)}\Bigr)^{SK},
  \qquad
  \frac{B}{\rhonu(\tau)}\ =\ \frac{4\sqrt{SK}\,\Nmax\max_{s,k}r_{s,k}}{\tau},
\end{equation}
$C$ is the number of classes and $\mathcal N(r;B_{\R^{SK}}(B))$ is the number of open balls of radius $r$ needed to cover the ball of radius $B$; consequently, with probability at least $1-\delta$ over $T$,
\[
  R(h_T)\ \le\ \widehat R(h_T)\ +\ \eps(T)\ +\ \sqrt{\frac{2\mathcal K\log2+2\log(1/\delta)}{m}}\,,
\]
where $\widehat R(h_T)=0$, every training point being its own nearest neighbor, and $\eps(T)=0$ on every training set that is separated at $\tau$.
\end{proposition}
\begin{proof}
Write $x_i=\Psii(M_i)$ for the training points, $\rho=\rhonu(\tau)$, and $x=\Psii(M)$ for an arbitrary module $M\in\mathcal R_N$; by Lemma~\ref{lem:bdd} all of these lie in the ball $B_{\R^{SK}}(B)$.
The partition constructed below depends on $\tau$, $\netcfg$ and $B$ only, not on $T$, as Theorem~3 of \citet{XuMannor2012} requires; the training set enters only through $\eps(T)$, which that theorem allows to depend on it.
If $T$ is not separated at $\tau$ the robustness condition holds trivially with $\eps(T)=1$, the loss being bounded by $1$, so assume it is.
The argument then has three steps: the classes are separated, so the $1$-NN rule is locally constant, so a fine enough partition has zero within-cell loss variation.

\emph{Separation.}
By Theorem~\ref{thm:cert}, $\norm{x_i-x_j}\ge\rho$ whenever $y_i\ne y_j$; in particular no two training modules with different labels share an embedding, so $h_T(x_i)=y_i$ for every $i$ and $\widehat R(h_T)=0$.

\emph{The $1$-NN rule is constant near each training point.}
Fix a training point $x_i$ and any $x$ with $\norm{x-x_i}<\rho/2$.
The nearest training neighbor of $x$---nearest among all $m$ training points, with no reference to any cell---is some $x_k$ with $\norm{x-x_k}\le\norm{x-x_i}<\rho/2$, hence
\[
  \norm{x_k-x_i}\ \le\ \norm{x_k-x}+\norm{x-x_i}\ <\ \rho,
\]
and separation forces $y_k=y_i$.
Every training point tied with $x_k$ at the minimal distance satisfies the same inequality, so the conclusion is independent of how ties are broken: $h_T(x)=y_i$ on the open ball of radius $\rho/2$ about $x_i$.

\emph{The partition, and why its cells have diameter $\rho/2$ rather than $\rho$.}
Cover $B_{\R^{SK}}(B)$ by $\mathcal N(\rho/4;B_{\R^{SK}}(B))$ open balls of radius $\rho/4$ and turn the cover into a partition by assigning each point to the first ball that contains it; every cell then has diameter less than $\rho/2$.
Robustness in \citet{XuMannor2012} is a property of a partition of the sample space, here $\Psii(\mathcal R_N)\times\{1,\dots,C\}$, and the $0/1$ loss reads the label, so the cells of that partition are the products of these $\mathcal N(\rho/4;B_{\R^{SK}}(B))$ cells with a single label: $\mathcal K=C\cdot\mathcal N(\rho/4;B_{\R^{SK}}(B))$ cells in all.
The volumetric bound $\mathcal N(r;B_{\R^{n}}(B))\le(1+2B/r)^{n}$ with $n=SK$ and $r=\rho/4$ gives the inequality in \eqref{eq:cover}.
Cells of diameter below $\rho$, which separation alone suggests, would not do: a test point $x$ and a training point $x_i$ in such a cell have $\norm{x-x_i}<\rho$, which bounds the distance from $x_i$ to the nearest neighbor of $x$ only by $2\rho$, and a training point of the other label can sit there.
Halving the diameter is exactly what closes the gap.

\emph{Zero within-cell variation.}
Let a training sample $(x_i,y_i)$ and a point $(x,y)$ of the sample space lie in the same cell.
Then $y=y_i$ and $\norm{x-x_i}<\rho/2$, so by the second step $h_T(x)=y_i=h_T(x_i)$, and both $0/1$ losses are $0$.
This is $(\mathcal K,\eps(T))$-robustness with $\eps(T)=0$ on separated training sets, and Theorem~3 of \citet{XuMannor2012}, with the loss bounded by $1$, gives the displayed bound.
The expression for $B/\rhonu$ substitutes Lemma~\ref{lem:bdd} and $\rhonu=\tau/(4\sqrt K)$.
\end{proof}
\takeaway{proved}{The two-sided distortion $B/\rhonu$ is what controls the generalization of the certified rule; proved, but numerically vacuous---the resulting number is astronomically large.}

The content is the middle term of \eqref{eq:cover}: the sample complexity is governed by $B/\rhonu$, the ratio of the upper gauge to the lower one---that is, by the two-sided distortion of the embedding, which is precisely what this paper computes and no prior multiparameter vectorization can.
Improving either gauge improves the bound, and a method with only an upper bound cannot state it at all.

We state its numerical status bluntly, because it would be easy to oversell.
$\mathcal K$ is exponential in $SK$ (here $1{,}600$), so the right-hand side is astronomically larger than $1$ and the bound is \emph{vacuous as a number}: it is a structural statement about which quantity controls generalization, not a usable risk estimate.
The exponent is the ambient embedding dimension; a non-vacuous version needs the covering number of the \emph{image} $\Psii(\mathcal R_N)$, whose intrinsic dimension we do not control.
This is the same $N$- and $K$-dependence the field already concedes as intrinsic---no quantitative embedding of the space of modules into a fixed-dimensional Euclidean space exists---and we prefer to state it than to bury it.

Nor is the vacuity an artifact of the ambient dimension.
Transporting a cover of the module space through the upper gauge---a cover of $\mathcal R_N$ in $\dmatS$ at scale $\rhonu/(4\sqrt S\,\Nmax)$ pushes forward, by Proposition~\ref{prop:upper}, to a cover of the image at scale $\rhonu/4$---replaces the exponent $SK$ in \eqref{eq:cover} by $2\Nmax S$, the number of bar coordinates, with $K$ demoted to the base; that is the one inequality in which both gauges act at once, and it is still astronomically large.
The cause is the scale, not the dimension, and it is measurable.
A robustness bound is informative only when the data concentrate on $\mathcal K\ll m$ cells of diameter $\rhonu$, and the certificate-1NN sweep of Section~\ref{sec:exp} finds the median distance from a test module to its nearest training module to be $6$--$12$ times $\rhonu$ at every sample size and at both configurations---and training modules, drawn from the same distribution, sit at the same scale from one another.
No two training modules share a cell: the empirical covering number of the training set at the certificate's scale is $m$ itself, and the complexity term is at least $\sqrt{2\log2}>1$ whatever the sample size.
This is the local/global mechanism of Section~\ref{sec:exp} in learning-theoretic form---the certificate lives below the sampling resolution of the data---and the same fact that empties the certified rule's coverage empties its risk bound.

\paragraph{RBF-SVM head.}
The same floor certifies the higher-accuracy Gaussian-SVM head.
Let $\Kern(M,M')=\exp(-\norm{\Psii(M)-\Psii(M')}^2/2h^2)$ with feature map $\iota$ into its RKHS $\Hilb$, so $\norm{\iota(M)-\iota(M')}_{\Hilb}^2=2(1-\Kern(M,M'))$.
(The bandwidth is $h$ and the kernel $\Kern$, a Greek kappa matching its ceiling $\kappaminus$ and distinct from the landmark index $k$; $\sigma$ stays reserved for simplices.)

\begin{lemma}[Kernel separation floor]\label{lem:rbf}
Fix $h>0$ and let $(M,M')$ be a $\netcfg$-coherent cross-class pair with $\dmatS(M,M')\ge\tau$, as in Theorem~\ref{thm:cert}.
Then
\[
  \Kern(M,M')\ \le\ \kappaminus(\tau):=\exp\!\Big(-\tfrac{\rhonu(\tau)^2}{2h^2}\Big)<1,
  \qquad
  \norm{\iota(M)-\iota(M')}_{\Hilb}\ \ge\ \sqrt{2\,(1-\kappaminus(\tau))}\ >\ 0 .
\]
\end{lemma}
\begin{proof}
$x\mapsto e^{-x^2/2h^2}$ is decreasing, so~\eqref{eq:cert} gives the bound on $\Kern$; substitute into $\norm{\iota(M)-\iota(M')}_{\Hilb}^2=2(1-\Kern)$.
As $\rhonu(\tau)>0$ we have $\kappaminus(\tau)<1$.
\end{proof}
\takeaway{proved}{The floor survives the Gaussian kernel: coherent cross-class pairs stay separated in the RKHS by an explicit amount.}

Monotonicity alone carries the linear floor $\rhonu$ through the kernel: coherent cross-class pairs separated at scale $\tau$ stay separated in the RKHS, by an explicit amount.
We are deliberately careful about what this does \emph{not} say.
A uniform lower bound on \emph{pairwise} cross-class distances does not lower-bound the margin of a max-margin separator, which is half the distance between the two \emph{convex hulls}; interleaved configurations have large pairwise distances and arbitrarily small margin.
Lemma~\ref{lem:rbf} therefore certifies the separation the kernel sees---the quantity that governs nearest-centroid and nearest-neighbor decisions in $\Hilb$---not the margin of the soft-margin SVM actually fitted, whose bandwidth $h$ is moreover selected per fold.
Closing that gap, so that the guarantee attaches to the head that attains the reported accuracy, is what Section~\ref{sec:exp} measures---and finds no Lipschitz route can do.
The bound relaxes as $h\to\infty$ and tightens as $h$ shrinks.

Nothing in Lemma~\ref{lem:rbf} is specific to the Gaussian kernel beyond monotonicity, and the coordinates of $\Psii$ are sums of hat functions---non-negative and histogram-like, the regime where $\chi^2$ and Hellinger kernels are classically competitive.
The floor carries to that whole family, with constants supplied by the same boundedness that powers Lemma~\ref{lem:bdd}.

\begin{lemma}[The floor survives non-Euclidean kernels]\label{lem:kersel}
Write $B_\infty:=\max_{s,k}\omega_s w_{s,k}\Nmax r_{s,k}$ for the per-coordinate bound of the stack (from the proof of Lemma~\ref{lem:bdd}; $B_\infty\le\Nmax\max_{s,k}r_{s,k}/\sqrt K$ under equal weights), and for a divergence $D\ge0$ with $D(x,x)=0$ let $\Kern_D(M,M')=\exp\bigl(-\gamma\, D(\Psii(M),\Psii(M'))\bigr)$, $\gamma>0$.
If $D$ dominates the Euclidean geometry on the embedding's range,
\begin{equation}\label{eq:divfloor}
  D(x,y)\ \ge\ c_D\,\norm{x-y}_2^2
  \qquad\text{for all } x,y\in[0,B_\infty]^{SK},
\end{equation}
then for every $\netcfg$-coherent cross-class pair with $\dmatS\ge\tau$,
\[
  \Kern_D(M,M')\ \le\ \exp\!\bigl(-\gamma\,c_D\,\rhonu(\tau)^2\bigr)<1,
  \qquad
  \norm{\iota_D(M)-\iota_D(M')}_{\Hilb_D}
  \ \ge\ \sqrt{2\bigl(1-e^{-\gamma c_D\rhonu(\tau)^2}\bigr)}\ >\ 0,
\]
whenever $\Kern_D$ is positive definite.
The three heads at issue satisfy \eqref{eq:divfloor} with explicit constants: the Gaussian kernel ($D=\norm{\cdot}_2^2$, $c_D=1$, $\gamma=1/2h^2$, recovering Lemma~\ref{lem:rbf}); the exponential $\chi^2$ kernel, $D(x,y)=\sum_i(x_i-y_i)^2/(x_i+y_i)$, with $c_D=1/(2B_\infty)$, since every coordinate obeys $x_i+y_i\le2B_\infty$; and the Hellinger kernel, $D(x,y)=\norm{\sqrt x-\sqrt y}_2^2$, with $c_D=1/(4B_\infty)$, since $|\sqrt{x_i}-\sqrt{y_i}|=|x_i-y_i|/(\sqrt{x_i}+\sqrt{y_i})\ge|x_i-y_i|/(2\sqrt{B_\infty})$.
(Coordinates with $x_i=y_i=0$ contribute zero to both sides, so the degenerate case is harmless; positive definiteness holds for all three---the $\chi^2$ distance is conditionally negative definite on the non-negative orthant, so $\exp(-\gamma\,\chi^2)$ is positive definite for every $\gamma>0$ by Schoenberg's theorem, and the Hellinger kernel is a Gaussian precomposed with the coordinatewise square root.)
\end{lemma}
\begin{proof}
Theorem~\ref{thm:cert} gives $\norm{\Psii(M)-\Psii(M')}_2\ge\rhonu(\tau)$ on coherent pairs; \eqref{eq:divfloor} and the monotonicity of $t\mapsto e^{-\gamma t}$ give the kernel ceiling, and $D(x,x)=0$ normalizes the diagonal, so $\norm{\iota_D(M)-\iota_D(M')}^2_{\Hilb_D}=2\bigl(1-\Kern_D(M,M')\bigr)$ as before.
\end{proof}
\takeaway{proved}{The same holds for the $\chi^2$ and Hellinger kernels with explicit constants, so choosing the kernel by cross-validation never loses the certificate.}

\begin{remark}[Certified kernel selection]\label{rem:kersel}
Lemma~\ref{lem:kersel} makes the \emph{choice} of kernel a certified degree of freedom: selecting among the Gaussian, $\chi^2$ and Hellinger heads by inner cross-validation on the training fold---exactly as the bandwidth and the SVM constant are already selected---changes only the constant $c_D$ in the floor, never its existence.
Whether the non-Euclidean heads pay in accuracy is an empirical question; what the lemma settles is that no choice in this family leaves the certificate behind.
\end{remark}

\section{Experiments}\label{sec:exp}

We test one question: does the closed-form slice-stack, with no learned embedding and no held-out calibration, match trained multiparameter baselines?
All numbers below are produced by the released slice-stack pipeline (the \texttt{experiments/} pipeline: bifiltration $\to$ fixed near-diagonal net $\to$ per-slice PALACE embedding $\to$ concatenation $\to$ classifier), averaged over three seeds.

\paragraph{Datasets.}
Orbit5k ($5{,}000$ point clouds, $5$ classes) and Orbit100k ($100{,}000$ clouds, the large-scale variant used by the transformer baselines), together with the TU graph benchmarks COX2, DHFR, MUTAG, NCI1, NCI109, PTC\_MR, PROTEINS, DD, IMDB-B\textsc{inary} and IMDB-M\textsc{ulti}.
Point clouds carry the alpha--DTM bifiltration (geometric scale $\times$ DTM codensity); graphs carry a sublevel bifiltration (heat-kernel signature $\times$ degree, and $\times$ node label for the labeled molecular sets).

\paragraph{Protocol (no feature tuning).}
The slice net~\eqref{eq:net} is the fixed near-diagonal fan ($S=8$ slices, half-width $a=0.6$, all base points $\beta_s=0$); each slice is PALACE-embedded with budget $K=200$, class-aware FPS landmarks, equal weights $w_k=K^{-1/2}$, the radius factor $\alpha=4$ (PALACE's published Orbit value, adopted unchanged; Remark~\ref{rem:radcap}), the top $N=50$ bars, essential classes kept as bars whose death is capped at the largest finite death on that slice rather than dropped (capping at the top of the filtration inflates $L$ and drowns every informative bar), and $\tau$ at the median half-persistence.
Three scales share the name $\tau$ in this paper, and we say which is meant at each use: the \emph{configuration} scale just named (PALACE's radius band, a property of one slice's diagrams), the \emph{audit} scale (the median cross-class $\dmatS$, at which coherence and the floor are instantiated), and the \emph{strict} certificate scale of \eqref{eq:taustar} (the minimum cross-class $\dmatS$, under which every cross-class pair satisfies the theorem's hypothesis by construction).
On Orbit5k the headline additionally uses PALACE's filtration multiplicity: three DTM bifiltrations ($k\in\{10,15,20\}$) concatenated at $K=200$ per block.
The only dataset-level constant, the axis calibration, is a pooled $q_{0.9}$ quantile of the two grades, computed unsupervised before any split.
We classify the unweighted stack $\Psii^\circ$ of Remark~\ref{rem:omega} by (a) nearest centroid (with the certificate of Theorem~\ref{thm:certified}) and (b) a $q$-tuned RBF-SVM (Lemma~\ref{lem:rbf}); at Orbit100k scale the exact kernel is replaced by a Nystr\"om feature map, calibrated against the exact head in the footnote of Table~\ref{tab:headline}.
The \emph{embedding} has no hyperparameter tuned on our data---net, landmarks, weights, $\tau$ and $N$ are fixed by the closed-form rule, and the radius factor and filtration multiplicity are PALACE's published values, adopted unchanged (our sweep in Remark~\ref{rem:radcap} confirms rather than selects them)---but the classifier head is not tuning-free: following PALACE, the RBF bandwidth (a quantile $q$ of training-fold distances) and $C$ are chosen by inner $3$-fold cross-validation on the training fold.
No held-out data is used, and the baselines we re-run are afforded the same courtesy.
Evaluation is $10$-fold stratified cross-validation ($3$-fold at Orbit100k); the embedding (landmarks, $\tau$, net) is frozen from the training fold, so there is no held-out tuning split.

\paragraph{Cost (the price of ``training-free'').}
On $16$ CPU cores---no GPU anywhere in the pipeline---the full Orbit5k experiment (diagrams for all $5{,}000$ clouds, per-fold landmarks, embedding and the tuned head) runs in $6.8$ minutes end to end: $1.2$ minutes for the first fold including all diagram computation, then embedding and head per fold (\texttt{timing.py}).
One Orbit100k fold costs $166$ minutes, $149$ of them in the Nystr\"om head at $4{,}000$ components, so the full $3$-fold protocol is ${\approx}8$ hours on the same $16$ cores.
For scale, D-\textsc{Gril} reports $5$--$30$ minutes \emph{per fold} of GPU training on the far smaller TU sets (their Table~4, on an A10), and the transformer baselines train for hours on accelerators before their first prediction.

\paragraph{Baselines.}
For point clouds: the sliced-Wasserstein~\citep{Carriere2017} and Fisher~\citep{LeYamada2018} kernels, \textsc{PersLay}~\citep{Carriere2020} (whose table is also the source of the two kernel rows), the Euler-characteristic surface with gradient-boosted trees (ECS$+$XGB, tuned to its best resolution), the single-parameter certified kernel PALACE, and the transformers Persformer and xPerT.
For graphs: multiparameter persistence landscapes (mp-l) and images (mp-i), and the \textsc{Gril} landscape---the trained or hyperparameter-tuned methods COMPLEX aims to match with no learned embedding.

\paragraph{Baseline protocol.}
Because the tables mix literature numbers with re-runs, we state per baseline which is which.
\emph{Literature, own protocol} ($^\ast$ in the tables): the sliced-Wasserstein and Fisher kernels, \textsc{PersLay} ($70/30$ splits, $100$ runs), xPerT and Persformer (their own splits and trained heads), PALACE ($10$-fold), graphcode (Orbit: $70/30$, $20$ and $10$ runs; graphs: $80/20$, $20$ runs), \textsc{Gril} together with the mp-i and mp-l rows it re-tabulates ($5$ stratified train/test splits, XGBoost head), TopoGCL ($10$-fold), and the ECS headline of~\citet{HacquardLebovici2024}.
\emph{Re-run on our clouds and folds}: ECS (matched bifiltration, best resolution, XGB), and in the contamination sweep ECS, ECC, HT and PALACE-1p, each scored under both a gradient-boosted and an RBF-SVM head and reported at the better, so that no baseline is handicapped by a head it was not designed for.
End-to-end architectures are reported natively, since forcing an SVM head would handicap them below their design.
Where a re-run and a literature number differ (ECS: $89.5$ against $91.8$) both are shown and the comparison is drawn against the literature figure.

\paragraph{Headline (point clouds).}
Table~\ref{tab:headline} reports accuracy on Orbit5k and Orbit100k.
On Orbit5k COMPLEX reaches $91.95\pm0.11$ over seeds ($\pm1.3$ over folds), using PALACE's published Orbit recipe lifted unchanged to the slice-stack---radius factor $\alpha=4$ and three DTM bifiltrations ($k\in\{10,15,20\}$) concatenated at $K=200$ per block, so the embedding is three times wider ($4{,}800$ against $1{,}600$); the single-bifiltration stack at the same $\alpha$ reads $91.74\pm0.10$, and $91.2$ at the $\alpha=1.75$ default (the radius sweep, monotone in $\alpha$, and the certified capped variant are priced below).
Multiplicity pays only with the extra width: the same three bifiltrations at a \emph{matched} total budget ($K=60$ per block, $1{,}440$ dimensions) read $91.53\pm0.01$, below the single bifiltration---so the $+0.2$ is bought by dimension, exactly as in PALACE's published recipe, and not by the second and third filtrations per se.
That places the closed-form embedding ahead of every diagram-based kernel and vectorization (\textsc{PersLay} $87.7$, Fisher $85.9$, sliced Wasserstein $83.6$), of xPerT, and of the strongest one-parameter methods, Persformer ($91.2\pm0.8$)\footnote{A reproducibility caveat attaches to the Persformer rows: the same table reports $99.1$ when the model reads the \emph{raw point clouds}, alone or alongside the diagrams---an input that is no longer persistence (its saliency-filtering experiment reads $91.1$)---while xPerT's independent re-run of the diagram-input model failed to train ($28.2\pm7.3$; ``the model was not trainable in both ORBIT5K and ORBIT100K''~\citep{xPerT}).
We quote the commonly cited raw accuracy and flag that the transformer rows carry more reproduction risk than their variances suggest.} and its own predecessor PALACE ($91.3\pm1.0$)---and, for the first time in this line of work, of the Euler-characteristic surface ($91.8\pm0.4$), previously the accuracy ceiling among training-free descriptors.
Two comparisons still deserve care.
The ECS lead we overturn is $0.15$ points, within one fold's deviation, so \emph{level} is the defensible reading and \emph{ahead} the nominal one; and ECS is not a persistence method---it bypasses homology entirely, which is exactly its advantage in cost and its limitation in what it can certify.
Within its own family COMPLEX has one published competitor: graphcode~\citep{KerberRussold2024}, a GNN trained on a two-parameter barcode hierarchy, reads $88.5\pm1.1$ on Orbit5k under its 70/30 protocol---$3.5$ below COMPLEX---and $92.3\pm0.3$ at the Orbit100k scale, previously the strongest reproducible number there (see below).
Two readings of the PALACE comparison are possible and the paper owes both.
Against PALACE's \emph{published} $91.3$, obtained under its own protocol, the multiparameter construction is ahead ($91.95$ against $91.3$)---using PALACE's own radius and filtration multiplicity, so the gap isolates the slice-stack.
Against a \emph{matched} re-run of the same single-parameter pipeline on the same clouds and the same folds---``PALACE-1p'' in Table~\ref{tab:contam}, where COMPLEX reads $90.9$ and PALACE-1p $87.2$ under one protocol---it is ahead by $3.7$ points.
The second is the like-for-like comparison and the first is not, so we report the published-protocol gap as the conservative claim and the matched gap as the informative one; comparing our re-run to another paper's headline in either direction is the asymmetry we are at pains to avoid elsewhere.
A slice-level ablation at the $\alpha=1.75$ default locates that gap: a \emph{single} diagonal slice of the bifiltration already reaches $89.1$ at the same total embedding dimension, and stacking to $S=8$ adds a further $+2.3$, saturating there ($S=16$ on a wider fan gains $+0.3$ over $S=4$; \texttt{slice\_value.py}, $5$-fold, one seed, so its $S=8$ figure of $91.4$ is not the $10$-fold $91.2$ of Table~\ref{tab:ablation}).
So the second parameter and the stack each contribute, and neither is redundant.
What the construction additionally buys is a module-level certificate and applicability to genuinely bifiltered data, with no learned embedding.
At the transformer scale the standing sharpens: on Orbit100k COMPLEX reads $92.98\pm0.02$ over three seeds ($\pm0.15$ over folds; $3$-fold, as the protocol states; single bifiltration, $\alpha=4$, Nystr\"om at $4{,}000$ components), ahead of xPerT ($91.1$), of Persformer ($92.0$, with the reproduction caveat above), and of the GNN-trained graphcode ($92.3\pm0.3$), the strongest previously published number at this scale---still with no learned embedding, subject to the approximation caveat in the caption.
The certified capped variant reads $92.78\pm0.20$ here---the certificate's construction is itself ahead of every published number at this scale (Remark~\ref{rem:radcap}).
At the $\alpha=1.75$ default with $1{,}000$ Nystr\"om components the same pipeline reads $92.3$; the gain decomposes cleanly, $+0.4$ from provisioning the head at its calibrated optimum and $+0.2$ from the radius factor, both fixed on Orbit5k before touching Orbit100k.

\begin{table}[t]
\centering
\caption{Point-cloud accuracy ($\%$, mean$\pm$std over folds).
COMPLEX's embedding is closed-form and training-free (its SVM head is cross-validated); baselines are trained or tuned.
Methods are grouped by family: the Euler-characteristic descriptors compute no persistence at all, and graphcode---GNN-trained, $70/30$ split---is the only other \emph{multiparameter} method reporting on these benchmarks.
Bold $=$ best in family.
$^\ast$ literature numbers, under each method's own protocol.
The COMPLEX RBF row uses the radius factor $\alpha=4$---PALACE's published Orbit value, adopted unchanged (Remark~\ref{rem:radcap}; the sweep and the certified capped variant are priced in the text)---with, on Orbit5k, PALACE's three-bifiltration recipe (DTM $k\in\{10,15,20\}$, $K=200$ per block); the nearest-centroid row is at the $\alpha=1.75$ default.
$^\dagger$Orbit100k uses a Nystr\"om feature map in place of the exact kernel (a $90\text{k}\times90\text{k}$ kernel is ${\sim}65$\,GB), provisioned at $4{,}000$ components.
Calibrated against the exact kernel on Orbit5k at $5$ folds and the headline $\alpha=4$ (\texttt{nystrom\_calib.py}; exact head $91.58\pm1.44$), the approximation reads $-0.48/-0.24/+0.04$ points at $250/500/1{,}000$ components and $+0.30/+0.36$ at $2{,}000/4{,}000$, where it acts as a regularizer rather than an approximation (at $\alpha=1.75$ the ladder is $-0.74/-0.04/+0.34/+0.52/+0.44$; the shape is the same).
The row is $3$ seeds $\times$ $3$ folds, the protocol stated in the text; at the $10$-fold protocol used elsewhere in this paper the same configuration reads $93.10\pm0.28$ (seed $0$).
At the $\alpha=1.75$ default with $1{,}000$ components the pipeline reads $92.28\pm0.02$ at $3$ folds and $91.94\pm0.35$ at $10$.}
\label{tab:headline}
\begin{tabular}{lcc}
\toprule
Method & Orbit5k & Orbit100k \\
\midrule
\multicolumn{3}{l}{\emph{One-parameter persistence}}\\
\quad SW kernel$^\ast$        & $83.6\pm0.9$ & -- \\
\quad Fisher kernel$^\ast$    & $85.9\pm0.8$ & -- \\
\quad \textsc{PersLay}$^\ast$ & $87.7\pm1.0$ & $89.2\pm0.3$ \\
\quad xPerT (transformer)$^\ast$      & $87.0$ & $91.1$ \\
\quad Persformer (transformer)$^\ast$ & $91.2\pm0.8$ & $\mathbf{92.0\pm0.4}$ \\
\quad PALACE$^\ast$ & $\mathbf{91.3\pm1.0}$ & -- \\
\midrule
\multicolumn{3}{l}{\emph{Euler characteristic (homology-free)}}\\
\quad ECS$+$XGB$^\ast$        & $\mathbf{91.8\pm0.4}$ & -- \\
\quad ECS$+$XGB, our re-run (matched bifiltration) & $89.5\pm0.9$ & -- \\
\midrule
\multicolumn{3}{l}{\emph{Multiparameter persistence}}\\
\quad Graphcode (GNN)$^\ast$~\citep{KerberRussold2024} & $88.5\pm1.1$ & $92.3\pm0.3$ \\
\quad COMPLEX, nearest-centroid (certified) & $73.4\pm2.6$ & $71.6\pm0.5$ \\
\quad COMPLEX, RBF-SVM (uncertified)         & $\mathbf{91.95\pm1.3}$ & $\mathbf{92.98\pm0.15}^\dagger$ \\
\bottomrule
\end{tabular}
\end{table}

\paragraph{The radius cap, priced.}
Remark~\ref{rem:radcap} leaves two variants; here is what the certified one costs.
Enforcing \eqref{eq:radcap} after construction reads $91.25\to90.24$ at $\alpha=1.75$ and $91.74\to91.31$ at $\alpha=4$ on Orbit5k (single bifiltration, $10$-fold RBF-SVM), and $92.98\to92.78\pm0.20$ on Orbit100k ($3$-fold, seed $0$).
The radius sweep is monotone in $\alpha$ for both variants ($91.25/91.69/91.74$ uncapped at $\alpha=1.75/2.5/4$), so certification and the radius factor point the same way, and the certified construction itself stays above every previously published number at the $100$k scale, graphcode's $92.3$ included.
The lost accuracy localizes where the unconstrained radii earn their accuracy: coordinates that read persistence mass \emph{near the diagonal}, a small-bar count signal the certified construction deliberately excludes.

\paragraph{The block weights, priced.}
Remark~\ref{rem:omega} leaves the choice of stack to accuracy, so we ran both at the head that carries it.
On the single bifiltration the weighted stack of Definition~\ref{def:embed} reads $91.79\pm0.12$ against the unweighted $91.74\pm0.10$; on the three-bifiltration headline $91.99\pm0.07$ against $91.95\pm0.11$; on Orbit100k $92.99$ against $92.96$ ($3$-fold, seed $0$, $\alpha=4$, Nystr\"om at $4{,}000$).
All three are ties within seed noise, in the weighted stack's favour by a hair, so the reported numbers would not move if the pipeline weighted its blocks (\texttt{omega\_sweep.py}).
The certified nearest-centroid head is the one that notices, preferring the unweighted stack by $1.5$ points ($73.87\pm0.38$ against $72.39\pm0.18$ at $\alpha=4$): it decides by Euclidean distances in the stack directly, and the weights shrink the extreme slices---the most class-discriminative ones---by up to $3.5\times$ relative to the diagonal, while the kernel head reabsorbs the same rescaling into its per-fold bandwidth.

\paragraph{Pooled degrees, priced.}
The Orbit numbers pool $H_0$ and $H_1$ into one multiset per slice before the top-$50$ cap, the shortcut Section~\ref{sec:embed-def} warns against; the certified statements attach to the per-degree stack.
On Orbit5k the per-degree stack reads within $0.1$ of the pooled one at matched embedding dimension ($91.54$ against $91.65$, $K=100$ per degree-block, $5$-fold, three seeds, $\alpha=1.75$) and $0.3$ below it at matched per-block $K$ and twice the dimension, so the shortcut buys no accuracy and the headline is not resting on it (\texttt{dimsep\_pointcloud.py}).
Under a shared cap the two regimes see different bar counts ($399$ against $747$ per cloud), which is the only respect in which they differ.

\paragraph{The slice weights, priced.}
Definition~\ref{def:embed} multiplies block $s$ by $\omega_s$; the classifier stacks the blocks unweighted, so the accuracy tables describe the plain concatenation while the certificate statements---and every audit in this section---describe the weighted stack.
The gap is measured and small.
On the headline configuration the weighted stack reads $91.99\pm0.07$ against $91.95\pm0.11$ on Orbit5k and $92.99$ against $92.96$ on Orbit100k at seed $0$ (\texttt{omega\_sweep.py}), inside seed noise both times; the certified nearest-centroid head is the one that notices, losing $1.5$ points at $\alpha=4$ ($72.39\pm0.18$ against $73.87\pm0.38$; Table~\ref{tab:headline}'s centroid row is at the $\alpha=1.75$ default), since a centroid rule reads the raw Euclidean geometry that the weights rescale by up to $3.5\times$ across the fan.
Nor does the theory turn on it: dropping the weights multiplies the upper gauge of Proposition~\ref{prop:upper} by $1/\omega_{\min}$ ($5.4$ for the default fan) and \emph{improves} the floor of Theorem~\ref{thm:cert} by $1/\max_s\omega_s$ ($1.6\times$), because the witnessing block is no longer shrunk---the two-sided bound survives with both constants explicit.
We report the unweighted numbers because every table in this paper is computed that way.

\paragraph{Robustness to outlier contamination.}
A codensity axis exists to suppress outliers, so the sharpest test of the second parameter is a dirty-data sweep against the descriptors that do not have one.
Table~\ref{tab:contam} replaces each point of each Orbit5k cloud, independently with probability $p$, by a uniform draw from the ambient box, and scores every method on the same clouds and the same folds.
The Euler descriptors are given their best resolution ($40\times40$; at the $20\times20$ default ECS loses $3.2$ points and the comparison is not a fair one), and each is reported at the better of gradient-boosted trees and an RBF-SVM head.

Two readings, and they differ.
Against a \emph{matched single-parameter} pipeline the second parameter does exactly what it is designed to do: the margin over PALACE-1p grows from $+3.7$ on clean data to $+12.4$ at $p=0.3$, and PALACE-1p has the worst degradation in the table ($-37.8$).
The same holds against the Euler characteristic curve (ECC, $-36.7$) and the hybrid transform (HT, $-30.1$).
Against \emph{ECS}, it does not.
COMPLEX leads by $1.8$ points on clean data, the lead is gone by $p=0.1$, and at $p=0.2$ and $p=0.3$ ECS is ahead by $0.2$ and $0.5$ points; ECS also degrades less overall ($-26.7$ versus $-29.1$).
We had predicted the opposite---the Euler characteristic is a signed alternating sum, so outlier simplices should cancel unpredictably rather than be suppressed---and the prediction is not borne out.
The honest statement is that on contaminated Orbit5k COMPLEX and ECS are empirically equivalent, with COMPLEX slightly ahead when clean and slightly behind when heavily contaminated.

This is the same concession Table~\ref{tab:headline} already makes on clean accuracy, and it has the same answer.
ECS matches COMPLEX here without computing homology at all, which is what makes it cheap---and is also why it can certify nothing: there is no module, no interleaving, and no distortion bound to instantiate.
What COMPLEX buys over the Euler family on this benchmark is not accuracy and not robustness, but the per-prediction guarantee of Section~\ref{sec:cls} and applicability to genuinely bifiltered data.

\begin{table}[t]
\centering
\caption{Outlier contamination on Orbit5k (accuracy $\%$, mean$\pm$std over $3$ seeds, $5$-fold CV, $1{,}000$ clouds per class).
Each point is replaced by a uniform draw with probability $p$.
Euler descriptors use a $40\times40$ surface and the better of XGB and RBF-SVM. \textsc{drop} is the accuracy lost from $p=0$ to $p=0.3$; smaller magnitude is more robust.
Bold $=$ best in column.}
\label{tab:contam}
\begin{tabular}{lcccccc}
\toprule
Method & $p{=}0$ & $0.05$ & $0.1$ & $0.2$ & $0.3$ & \textsc{drop} \\
\midrule
COMPLEX      & $\mathbf{90.9\pm0.0}$ & $\mathbf{83.8\pm0.4}$ & $\mathbf{80.7\pm0.2}$ & $72.5\pm0.3$ & $61.9\pm0.2$ & $-29.1$ \\
ECS          & $89.1\pm0.2$ & $83.4\pm0.2$ & $80.6\pm0.3$ & $\mathbf{72.7\pm0.2}$ & $\mathbf{62.4\pm0.2}$ & $\mathbf{-26.7}$ \\
HT           & $86.5\pm0.1$ & $76.8\pm0.0$ & $73.7\pm0.2$ & $65.0\pm0.3$ & $56.4\pm0.2$ & $-30.1$ \\
ECC          & $83.6\pm0.1$ & $72.3\pm0.1$ & $65.7\pm0.3$ & $55.6\pm0.1$ & $46.9\pm0.3$ & $-36.7$ \\
PALACE-1p    & $87.2\pm0.1$ & $78.1\pm0.1$ & $71.3\pm0.3$ & $61.4\pm0.1$ & $49.4\pm0.3$ & $-37.8$ \\
\bottomrule
\end{tabular}
\end{table}

\paragraph{Graphs.}
Table~\ref{tab:graph} reports the TU benchmarks.
The geometry-only bifiltration (HKS $\times$ degree) is already competitive with \textsc{Gril}; adding the node-label axis lifts every molecular set and takes COMPLEX past \textsc{Gril} on MUTAG, DHFR, and PROTEINS---using the same node information the trained baselines see, but with no training.
Concatenating three bifiltrations with the homology degrees kept separate---the ``multi'' column, and the per-degree form the certificate requires---is stronger still: \emph{one fixed configuration}, applied unchanged to every data set, exceeds \textsc{Gril} on all four shared benchmarks ($+1.9$ COX2, $+2.5$ DHFR, $+0.7$ MUTAG, $+1.3$ PROTEINS) and clears the majority baseline everywhere.
Enforcing the radius cap \eqref{eq:radcap}---the variant the certified statements attach to (Remark~\ref{rem:radcap})---costs at most $0.8$ points on this table and changes no ordering: over three seeds it reads $88.6\pm0.8$ on MUTAG, $80.0\pm0.3$ on DHFR, $81.6\pm0.4$ on COX2 and $71.4\pm1.3$ on PROTEINS, still above \textsc{Gril} on all four.
Against the multiparameter images and landscapes baselines---re-tabulated by \textsc{Gril} under its own protocol---the multi column is ahead on every shared set except DHFR, where mp-i leads it by $0.1$.
The GNN-trained graphcode~\citep{KerberRussold2024}, strongest at the Orbit100k scale, is weaker here: its published graph numbers ($86.4$ MUTAG, $78.7$ COX2, $76.4$ DHFR, $73.6$ PROTEINS, $65.4$ IMDB-B, under an $80/20$ split) sit below the multi column on every shared set except PROTEINS.
Selecting the bifiltration \emph{set} per data set---the ``sel.''\ column, chosen on the training folds only (Section~\ref{sec:dcomplex})---lifts exactly the sets whose best axes the fixed bank omits: $+3.7$ on NCI1 and $+1.7$ on DHFR, at a cost on the smaller sets.
Beyond the multiparameter family, trained topological pipelines report stronger numbers on several of these sets---the contrastive TopoGCL~\citep{ChenFriasGel2024} reads $81.3$ on NCI1, $90.1$ on MUTAG and $77.3$ on PROTEINS at the same $10$-fold protocol---and we do not claim otherwise; the comparison this table makes is with no learned embedding.
Two cells are worth naming anyway: on COX2 the multi configuration's $81.7$ exceeds TopoGCL's $81.5$, and on DHFR the selected configuration's $81.8$---and the fixed Gaussian head of Table~\ref{tab:kernelsel} at $82.2$---sits level with TopoGCL's $82.1$, a training-free embedding matching the strongest trained topological pipeline we know of on that set.
COX2 is a degenerate benchmark and we read nothing into it: its majority class is $78.2\%$ of the data, and all seven multiparameter methods tabulated by D-\textsc{Gril}~\citep{Mukherjee2024}---D-\textsc{Gril} itself included---report exactly $78.16\pm0.41$: the constant classifier.
Our own single-bifiltration $78.0$/$78.1$ sits at that rate, and \textsc{Gril}'s $79.8$ barely above it.
The entries that clear it by a real margin are all variants of this paper's construction---the multi configuration's $81.7\pm0.4$ (Table~\ref{tab:graph}, $+3.5$ over the constant classifier, more than twice \textsc{Gril}'s margin), SELECT-MULTI's $81.5$, and SELECT-COMPLEX's $80.1$ ($81.0$ under its inner-CV variant; Table~\ref{tab:select})---joined, outside the multiparameter family, by the trained TopoGCL at $81.5$~\citep{ChenFriasGel2024}.
On a benchmark where every tabulated multiparameter method reproduces the majority rate exactly, that is the result worth reporting---not the ranking.

\paragraph{Comparison with D-\textsc{Gril}.}
D-\textsc{Gril}~\citep{Mukherjee2024} is the closest point of comparison---a \emph{learned} two-parameter bifiltration---and reports MUTAG $85.09\pm5.99$, PROTEINS $69.45\pm4.11$, DHFR $61.24\pm4.37$ and IMDB-B\textsc{inary} $62.60\pm6.56$.
Most of that table is not comparable with ours: they standardize every multiparameter signature onto a $3$-layer MLP head and say so, their protocol is $5$-fold to our $10$, and on the labeled molecular sets their bifiltration collapses to the constant classifier---their DHFR row sits at $\approx\!61$ for \emph{all seven} methods, and DHFR's majority class is $60.98\%$, so three of those entries reproduce it to two decimals.
The original papers report $\approx\!81$ there, and our $+$label column reaches $80.6$; that gap is what a node-label axis opens.
Comparing our $80.6$ to their $61.2$ there would measure their feature set, not our embedding.

The comparison that does hold is IMDB-B\textsc{inary}, which carries no node labels, so neither method can differ in what it reads from the graph.
COMPLEX reaches $70.0$ against D-\textsc{Gril}'s $62.60$, with \textsc{Gril}, mp-l and the persistence-image baseline all at chance ($50.0$) and mp-i at $56.60$.
A closed-form, training-free embedding is $7.4$ points ahead of a gradient-trained bifiltration on the one benchmark where the two are measured on equal footing---and IMDB-B\textsc{inary} is also where capping essential classes helped most ($66.9\to70.0$), since in a social graph the unfilled cycles are the signal.

\paragraph{At OGB scale, the training-free embedding loses.}
The TU sets are small, and a fair reading of them has to ask whether the result holds where the data is larger and the split harder.
We therefore ran the graph pipeline unchanged on \texttt{ogbg-molhiv}~\citep{Hu2020OGB}---$41{,}127$ molecules, the official scaffold split, ROC-AUC on a $3.5\%$ positive rate---with the landmarks fit on the training split alone and two changes the benchmark forces: a Nystr\"om head with a linear SVM in place of the exact kernel, and ROC-AUC in place of accuracy (\texttt{graph\_ogb.py}, three seeds).
COMPLEX reads $72.92\pm0.41$ test ROC-AUC with the HKS$\times$degree bifiltration and $72.85\pm0.19$ with HKS$\times$label, against $76.1$ for a plain GCN, $75.6$ for GIN and $77.1$ for GIN with virtual nodes on the OGB leaderboard.
That is a loss of three to four points, and we report it as one.
Two honest readings are available and we decline to choose between them: the scaffold split is a deliberate distribution shift, which a fixed closed-form embedding has no mechanism to adapt to, and message passing on molecular graphs has a decade of architecture search behind it that a landmark embedding of a bifiltration does not.
Neither excuses the gap.
What the result bounds is the scope of the accuracy claim in this paper---training-free parity with trained multiparameter baselines is demonstrated on point clouds and on the TU benchmarks, not at OGB scale---while leaving the certificate, which no baseline in either regime supplies, untouched.

\begin{table}[t]
\centering
\caption{Graph accuracy ($\%$, $10$-fold CV; COMPLEX columns are mean$\pm$std over $3$ seeds).
``geom.'' $=$ HKS$\times$degree; ``$+$label'' adds the node-label axis; ``multi'' concatenates \emph{three} bifiltrations (HKS$_{10}\times$degree, HKS$_1\times$closeness, HKS$_{10}\times$label) with the homology degrees kept separate---the per-degree form Theorem~\ref{thm:cert} requires---at matched total dimension, mean$\pm$std over $3$ seeds.
The multi column is \emph{one fixed configuration} applied unchanged to every data set, not the best of several chosen per data set.
Bold $=$ best; $^\ast$ literature, under each method's own protocol and with a single bifiltration.
``sel.''\ is SELECT-MULTI (Section~\ref{sec:dcomplex}): the top-$5$ bifiltrations from the full $21$-pair bank, chosen per fold by inner cross-validation on the training fold only and concatenated---$k=5$ fixed across data sets, same protocol and total budget as ``multi''.
``majority'' is the constant-classifier rate.
It matters most on COX2, where every previously published entry sits within $1.6$ points of it and the single-bifiltration columns sit \emph{below} it; the multi and selected configurations are the only entries in that row to clear it by more than $3$ points.}
\label{tab:graph}
\small
\begin{tabular}{lccccccccc}
\toprule
& & \multicolumn{4}{c}{COMPLEX} & & & \\
\cmidrule(lr){3-6}
Dataset & majority & geom. & $+$label & multi & sel. & \textsc{Gril}$^\ast$ & mp-i$^\ast$ & mp-l$^\ast$ \\
\midrule
MUTAG    & $66.5$ & $88.1\pm1.1$ & $\mathbf{89.7\pm0.7}$ & $88.5\pm0.2$ & $88.8\pm0.4$ & $87.8$ & $85.6$ & $85.7$ \\
DHFR     & $61.0$ & $77.7\pm0.8$ & $80.6\pm0.4$ & $80.1\pm0.5$ & $\mathbf{81.8\pm1.9}$ & $77.6$ & $80.2$ & $79.5$ \\
COX2     & $78.2$ & $78.0\pm0.1$ & $78.1\pm0.2$ & $\mathbf{81.7\pm0.4}$ & $81.5\pm0.4$ & $79.8$ & $77.9$ & $79.0$ \\
PROTEINS & $59.6$ & $69.5\pm0.4$ & $71.3\pm0.5$ & $\mathbf{72.2\pm0.5}$ & $70.9\pm0.7$ & $70.9$ & $67.3$ & $65.8$ \\
NCI1     & $50.1$ & $69.5\pm0.1$ & $73.1\pm0.2$ & $74.4\pm0.1$ & $\mathbf{78.1\pm0.2}$ & -- & -- & -- \\
PTC\_MR  & $55.8$ & --     & --     & $57.0\pm1.7$ & $54.1\pm1.4$ & -- & -- & -- \\
IMDB-B   & $50.0$ & $\mathbf{70.0\pm0.4}$ & -- & -- & -- & -- & -- & -- \\
\bottomrule
\end{tabular}
\end{table}

\paragraph{Kernel head, selected per data set.}
Remark~\ref{rem:kersel} makes the kernel a certified degree of freedom; here is what it is worth.
Table~\ref{tab:kernelsel} scores the same five-bifiltration embedding under a fixed Gaussian head and under \emph{select}---the kernel chosen among Gaussian, $\chi^2$ and Hellinger inside the same train-only inner cross-validation that already tunes the bandwidth and $C$.
The coordinates of $\Psii$ are non-negative sums of hat functions, and the histogram kernels behave exactly as that suggests: selection is worth $+2.9$ on NCI1 and $+5.4$ on PTC\_MR (where $\chi^2$ is chosen in over half the folds), $+1.4$ on COX2 (Hellinger), is neutral on PROTEINS and IMDB-B, and costs about a point on DHFR and MUTAG, where the Gaussian was already the winner and the inner CV occasionally guesses wrong.
On the per-degree stack the effect repeats ($77.7\to79.6$ on COX2, $54.1\to58.3$ on PTC\_MR).
The lesson mirrors slice selection: adaptivity pays where the default is wrong, costs a little where it is right---and, by Lemma~\ref{lem:kersel}, none of it touches the certificate.

\begin{table}[t]
\centering
\caption{Kernel selection on the five-bifiltration stack (accuracy $\%$, $10$-fold CV, mean$\pm$std over $3$ seeds).
``Gaussian'' fixes the RBF head; ``select'' chooses among Gaussian, $\chi^2$ and Hellinger by inner CV on the training fold ($12$ bandwidth $\times$ $C$ candidates per kernel), which Remark~\ref{rem:kersel} shows is certificate-preserving.
``picks'' counts the inner-CV winner over the $30$ test folds.
This pipeline differs from Table~\ref{tab:graph} (multi-big bank at $K=40$ per block, no per-feature standardization), so only within-row comparisons are meaningful.}
\label{tab:kernelsel}
\begin{tabular}{lccl}
\toprule
Dataset & Gaussian & select & picks (of $30$) \\
\midrule
MUTAG    & $\mathbf{86.9\pm0.2}$ & $86.2\pm0.9$ & rbf $15$, Hellinger $9$, $\chi^2$ $6$ \\
COX2     & $77.2\pm0.6$ & $\mathbf{78.6\pm0.6}$ & Hellinger $19$, $\chi^2$ $9$ \\
DHFR     & $\mathbf{82.2\pm0.6}$ & $81.1\pm0.5$ & Hellinger $15$, rbf $10$ \\
PROTEINS & $73.3\pm0.2$ & $73.2\pm0.5$ & rbf $19$, Hellinger $8$ \\
NCI1     & $76.3\pm0.3$ & $\mathbf{79.2\pm0.2}$ & $\chi^2$ $17$, Hellinger $13$ \\
PTC\_MR  & $52.0\pm2.3$ & $\mathbf{57.4\pm0.2}$ & $\chi^2$ $16$, Hellinger $9$ \\
IMDB-B   & $66.0\pm0.1$ & $66.0\pm0.2$ & $\chi^2$ $20$, rbf $8$ \\
\bottomrule
\end{tabular}
\end{table}

\paragraph{Slice-net ablation.}
Table~\ref{tab:ablation} varies the only design constant, the net.
Two findings.
First, the near-diagonal fan decisively beats the axis-aligned percentile slices ($91.2$ vs $85.7$), confirming Section~\ref{sec:slicenet}.
Second, and central to the closed-form thesis: making the net \emph{data-adaptive} does not help.
The certified-adaptive net, which selects slices by a discrimination score, and even a gradient-trained net (angles optimized end-to-end through the persistence layer) both fall at or below the fixed fan and tend to overfit the selection.
(In a separate small-scale probe---$150$ training clouds, $6$ angles, persistence statistics in place of the landmark head---gradient-optimized angles scored $75.6$ on held-out clouds against the fixed fan's $77.2$; the absolute level is not comparable to the table, only the ordering.)
No net we tried improves on the fixed fan by more than $0.7$ points---inside the $1.5$-point fold standard deviation---so we read this as a null result rather than an optimality claim: the closed-form default costs no accuracy, and net design is not where the headroom lies.

\paragraph{Offset slices.}
Table~\ref{tab:ablation} varies the fan's angles; its base points are all at the origin, while \eqref{eq:slice} allows any point of the orthant boundary.
Here is what that generality is worth (\texttt{offset\_sweep.py}, $\alpha=4$, $10$-fold, three seeds, offsets in calibrated grade units, where the axis calibration puts both grades mostly in $[0,1]$).
The origin fan ($8$ angles, $K=200$) reads $91.74\pm0.10$, reproducing the single-bifiltration figure of the headline paragraph seed for seed under a harness that shares no code path with it.
Trading half the angles for offsets at the same block count and width ($4$ angles $\times\,\{-0.25,+0.25\}$) reads $91.59\pm0.04$; giving all eight angles a $\pm0.25$ pair reads $91.51\pm0.03$ at matched dimension ($24$ blocks, $K=65$, $1{,}560$ dimensions) and $91.65\pm0.02$ at matched per-block budget ($24$ blocks, $4{,}800$ dimensions, three times the width); five offsets per angle read $91.07\pm0.08$.
Every offset net sits at or below the origin fan and the loss grows with the number of offsets, and since even the three-times-wider variant does not close it, this is not a capacity effect.
The mechanism is visible in \eqref{eq:slicepar}: below its base point an offset slice reads one axis no further---for $\beta=(c,0)$ a simplex with scale under $c$ enters at its codensity alone---so at a fixed budget the offsets buy a coarser view of the bifiltration rather than a new one.
We keep them in the theory, where they are what makes the fibered barcode the full rank invariant and $\dmatS$ independent of the translation into the orthant, and report the origin fan as the default because it is what the data prefers---the same verdict this section returns for adaptive angles.

\begin{table}[t]
\centering
\caption{Slice-net ablation on Orbit5k ($10$-fold CV accuracy $\%$, mean over seeds).
Only the net changes; embedding and classifier are fixed at the $\alpha=1.75$ single-bifiltration default, so the absolute level sits below the headline configuration and only the ordering matters here.
$^\ddagger$The adaptive rows use a non-degeneracy proxy for the certifiability filter of Proposition~\ref{prop:select}, not the full $\rhonu\ge\rho_0$ admissibility test.}
\label{tab:ablation}
\begin{tabular}{llcc}
\toprule
Net & $(S,a)$ & nearest-centroid & RBF-SVM \\
\midrule
percentile (axis-aligned) & $(4,-)$    & $44.0$ & $85.7$ \\
near-diagonal fan         & $(4,0.6)$  & $72.6$ & $90.9$ \\
near-diagonal fan         & $(8,0.6)$  & $\mathbf{73.4}$ & $\mathbf{91.2}$ \\
near-diagonal fan         & $(8,0.3)$  & $68.6$ & $90.7$ \\
certified-adaptive$^\ddagger$ & $(4,0.7)$  & $67.9$ & $89.0$ \\
certified-adaptive        & $(8,0.7)$  & $68.8$ & $90.4$ \\
\bottomrule
\end{tabular}
\end{table}

\paragraph{Coherence audit.}
Theorem~\ref{thm:cert} is proved under witnessing-slice coherence, and Remark~\ref{rem:audit} promised its empirical rate; we report it here.
On Orbit5k we sample $2{,}000$ cross-class pairs, keep the $1{,}000$ whose sliced matching distance clears the median scale $\tau$, and test Definition~\ref{def:coh} on each pair's witnessing slice---over active landmarks only ($k\in\mathcal A_{s^\star}(\tau)$), where the weights cancel and the check reduces to $|\varphi_k(A)-\varphi_k(B)|\ge\db(A,B)/4$---the realized-distance form, stronger than Definition~\ref{def:coh}'s scale form, so the rates below lower-bound the rate of the definition itself.
Coherence holds on $97.5\%$ of pairs for the $S=8$ fan and $98.0\%$ for $S=4$ at the $\alpha=1.75$ default, while the certificate's \emph{conclusion}---embedded distance at least $\rhonu(\tau)$---holds on $100\%$ of pairs for both.
The residual $2.5\%$ at $\alpha=1.75$ is the expected gap, and the reason Remark~\ref{rem:audit} separates the two levels: coherence is sufficient, not necessary, so the separation it guarantees survives on the pairs where the hypothesis itself cannot be verified.
At the adopted radius factor $\alpha=4$ the hypothesis itself saturates: $100.0\%$ of $1{,}000$ audited pairs are coherent ($\rhonu=0.0220$ at $\tau=1.244$, the median of this $2{,}000$-pair sample; the $20{,}000$-pair estimate of Remark~\ref{rem:admisfail} is $1.220$), capped and uncapped alike---larger radii enlarge every active set, so the certificate's hypothesis is \emph{easier} to satisfy at exactly the configuration the accuracy prefers.

\paragraph{Numeric instantiation of the certificate.}
The theory promises a finite, reportable floor; here it is, on Orbit5k with the $S=8$ fan and $K=200$ ($\Psii$ into $\R^{1600}$), for the \emph{certified capped} variant at $\alpha=4$ (Remark~\ref{rem:radcap})---the construction the certified statements attach to.
All quantities are computed from the embeddings over $20{,}000$ sampled cross-class pairs, of which the $\approx\!10{,}000$ with $\dmatS\ge\tau$---the theorem's hypothesis---enter the distance rows (\texttt{certificate\_instantiation.py}).

\begin{center}
\begin{tabular}{lcl}
\toprule
quantity & value & \\
\midrule
$\tau$ (median $\dmatS$ over cross-class pairs) & $1.220$ & the audit scale \\
$\rhonu(\tau)=\tau/(4\sqrt K)$                  & $0.0216$ & the certified floor \\
$\min\norm{\Psii(M)-\Psii(M')}$, over kept pairs & $0.0950$ & $4.4\times\rhonu$ \\
median $\norm{\Psii(M)-\Psii(M')}$, over kept pairs & $0.2271$ & $10.5\times\rhonu$ \\
$\widehat\Delta$ (min centroid gap)             & $0.0584$ & $2.7\times\rhonu$ \\
$\max_c$ class radius in $\R^{1600}$            & $0.577$  & \\
\bottomrule
\end{tabular}
\end{center}

(The uncapped variant reads $6.3\times$ and $13.9\times$ with class radius $3.48$; at the $\alpha=1.75$ default the ratios were $2.5\times$ and $6.2\times$.
Larger radii move every realized distance up while the equal-weight floor stays $\tau/(4\sqrt K)$, so the ratio measures the pairing of floor and configuration, not the theory alone.)

Two readings, and the second is the one we did not expect.
The floor is \emph{tight to within a small factor}: the smallest embedded distance attained by any pair satisfying the hypothesis is $0.0950$ against a guaranteed $0.0216$---a factor $4.4$ here, and $2.5$ at the $\alpha=1.75$ default, where the realized distances sit lower.
For a bound routed through a functional Lebesgue number, a quarter-scale constant and an equal-weight $K^{-1/2}$, that is close.
(A minimum over sampled pairs can only fall as the sample grows; at the $\alpha=1.75$ default it moves from $2.7\times$ at $1{,}000$ pairs to $2.5\times$ at $20{,}000$, so the estimate is stable.
The median, which is not sensitive to the sample size in this way, is the more robust statistic and gives $10.5\times$.)
And the naive route from the certificate to the centroid guarantee is not merely false in principle (Section~\ref{sec:cls}) but quantitatively hopeless: the within-class radius that any pairwise-to-centroid bound must subtract is $0.577$, a factor $27$ above the floor, so the triangle bound $\widehat\Delta\ge d_{\min}-r_c-r_{c'}$ evaluates to $-1.06$.
Notably $\widehat\Delta=0.0584$ is itself $2.7\rhonu$---the centroids \emph{are} separated---but the certificate cannot be what proves it.

\paragraph{What the certified $1$-NN rule delivers, and why it is not more.}
Corollary~\ref{cor:nn} accepts a test module when $\eps<\rhonu(\tau)/2$.
We measured its coverage on Orbit5k over $K\in\{8,\dots,200\}$ and five choices of $\tau$ (\texttt{cert\_1nn.py}, $500$ train / $400$ test), first at the $\alpha=1.75$ default, where the mechanism is cleanest, and then---at the end of this section---at the headline configuration.
\textbf{Coverage never exceeds $1.5\%$ and is $0$ in most cells.}
We report this plainly because it bounds what the certificate currently buys, and because the reason is diagnostic rather than incidental.

It is \emph{not} that the floor is loose---we just measured it at $2.5\times$.
Replacing $\rhonu$ by the realized minimum cross-class distance, i.e.\ granting a perfectly tight certificate, leaves coverage below $1.5\%$ as well.
The binding constraint is that the median nearest-training-neighbor distance ($0.056$ at $K=200$) is several times the smallest cross-class distance, so a test module is typically no closer to its own class than the classes are to each other.
That is a statement about sampling density relative to class separation, and it also explains why $1$-NN on this embedding attains only $66.7\%$ at $n=500$ (rising to $78.0\%$ at $n=4{,}000$): the head the certificate reaches is not the head that is accurate.
Three further levers---the landmark budget, the sampling density and the metric---are swept in Appendix~\ref{app:sweep}; none moves the shortfall by more than a bounded factor, for the reasons summarized next.

\paragraph{The obstruction, stated.}
Four routes, four mechanisms, all measured on the same data (Figure~\ref{fig:obstruction} shows two): the floor cannot be tightened (it is within $2.5\times$ of realized, and a perfect floor leaves coverage below $1.5\%$), the landmark budget is invariant, sampling density is stationary, and the metric returns a bounded factor.
A fifth closes analytically.
The natural set-valued relaxation---accept the labels of every training module inside the certified ball---is sound, but that ball is nonempty exactly when $\eps<\rhonu/2$, so it fires on the same event and certifies a set in place of a label.
Exclusion runs the other way: ruling a class \emph{out} needs the Lipschitz \emph{upper} bound of Proposition~\ref{prop:upper}, which every multiparameter vectorization already has, and only inclusion reads $\rhonu$.

They share a cause, and it is visible in one measurement.
For a test module let $d_{(1)}$ and $d_{(2)}$ be the distances to the nearest training module of the nearest and of the runner-up class.
At $K=200$ the median ratio $d_{(2)}/d_{(1)}$ is $1.10$: the second class is ten per cent further away than the first.
Only $2.8\%$ of test modules have a runner-up even $1.5\times$ behind, and covering the true class at a $1.25\times$ radius takes $2.7$ of the $5$ labels.
The classes are interleaved at nearest-neighbor scale, and more landmarks make it worse, not better ($d_{(2)}/d_{(1)}$ falls from $1.19$ at $K=16$ to $1.10$ at $K=200$).

That single fact accounts for everything above.
Corollary~\ref{cor:nn} accepts when $\eps<\rhonu(\tau)/2$, and $\rhonu$ is at best a constant factor below the smallest embedded cross-class distance; with the classes $10\%$ apart at that scale, the acceptance ball contains a handful of near-duplicates and nothing else.
The same geometry caps $1$-NN at $78\%$, empties the set-valued rule, and would force conformal sets to more than half the label space.

Re-run at the headline configuration ($\alpha=4$, capped, three seeds), every number above survives---coverage $\le0.5\%$ at $K=200$, $d_{(2)}/d_{(1)}$ at $1.12$--$1.13$---with one change: under the radius cap the shortfall grows with the budget ($9\times$ at $K=8$ to $25\times$ at $K=200$) instead of staying invariant, so the budget is a lever \emph{against} the certificate there (Appendix~\ref{app:sweep}).

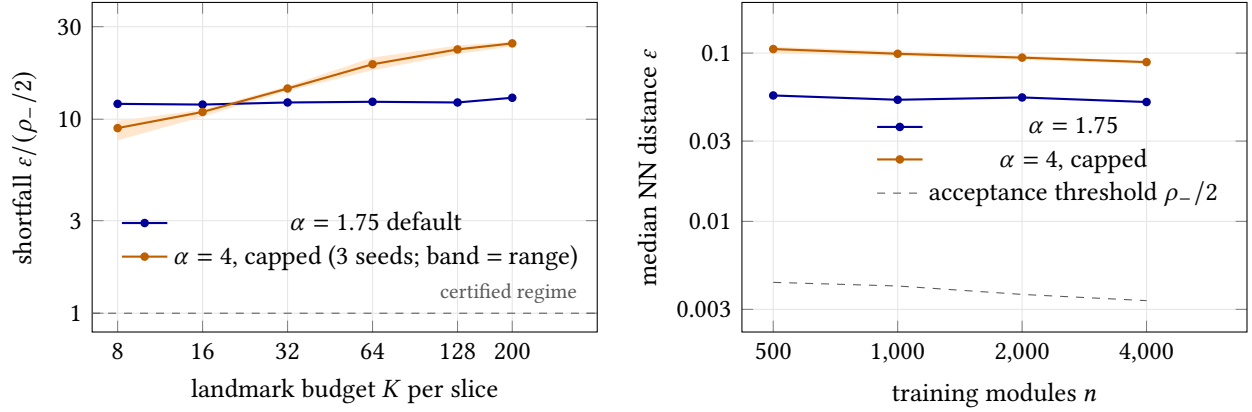
\begin{figure}[t]
\centering
\begin{tikzpicture}
\begin{axis}[width=0.5\linewidth, height=0.36\linewidth, at={(0,0)},
  xmode=log, log basis x=2, ymode=log,
  xlabel={landmark budget $K$ per slice}, ylabel={shortfall $\eps/(\rhonu/2)$},
  xtick={8,16,32,64,128,200}, xticklabels={8,16,32,64,128,200},
  ytick={1,3,10,30}, yticklabels={1,3,10,30}, ymin=0.8, ymax=40, xmin=6.5, xmax=400,
  grid=major, grid style={gray!20}, log ticks with fixed point,
  tick label style={font=\small}, label style={font=\small},
  legend style={font=\small, draw=none, fill=none, at={(0.04,0.16)}, anchor=south west}]
\addplot[fill=orange!20, draw=none, forget plot] coordinates
  {(8,9.8)(16,11.2)(32,14.8)(64,20.8)(128,24.0)(200,25.1)(200,23.7)(128,21.7)(64,17.9)(32,14.0)(16,10.3)(8,7.8)} --cycle;
\addplot[thick, color=blue!60!black, mark=*, mark size=1.2pt] coordinates
  {(8,12.0)(16,11.9)(32,12.2)(64,12.3)(128,12.2)(200,12.9)};
\addlegendentry{$\alpha=1.75$ default}
\addplot[thick, color=orange!75!black, mark=*, mark size=1.2pt] coordinates
  {(8,9.0)(16,10.9)(32,14.4)(64,19.2)(128,22.9)(200,24.6)};
\addlegendentry{$\alpha=4$, capped ($3$ seeds; band $=$ range)}
\addplot[dashed, color=gray!70!black] coordinates {(6.5,1)(400,1)}
  node[pos=0.98, above left, font=\scriptsize, color=gray!70!black] {certified regime};
\end{axis}
\begin{axis}[width=0.5\linewidth, height=0.36\linewidth, at={(0.52\linewidth,0)},
  xmode=log, log basis x=2, ymode=log,
  xlabel={training modules $n$}, ylabel={median NN distance $\eps$},
  xtick={500,1000,2000,4000}, xticklabels={500,{1,000},{2,000},{4,000}},
  ytick={0.003,0.01,0.03,0.1}, yticklabels={0.003,0.01,0.03,0.1}, ymin=0.0022, ymax=0.2,
  xmin=420, xmax=7000, grid=major, grid style={gray!20},
  tick label style={font=\small}, label style={font=\small},
  legend style={font=\small, draw=none, fill=none, at={(0.97,0.52)}, anchor=east}]
\addplot[fill=orange!20, draw=none, forget plot] coordinates
  {(500,0.1090)(1000,0.1025)(2000,0.0968)(4000,0.0897)(4000,0.0862)(2000,0.0913)(1000,0.0963)(500,0.0997)} --cycle;
\addplot[thick, color=blue!60!black, mark=*, mark size=1.2pt] coordinates
  {(500,0.0560)(1000,0.0528)(2000,0.0545)(4000,0.0512)};
\addlegendentry{$\alpha=1.75$}
\addplot[thick, color=orange!75!black, mark=*, mark size=1.2pt] coordinates
  {(500,0.1057)(1000,0.0991)(2000,0.0940)(4000,0.0884)};
\addlegendentry{$\alpha=4$, capped}
\addplot[dashed, color=gray!70!black] coordinates {(500,0.00434)(1000,0.00413)(2000,0.00368)(4000,0.00338)};
\addlegendentry{acceptance threshold $\rhonu/2$}
\end{axis}
\end{tikzpicture}
\caption{The two swept levers behind ``The obstruction, stated'' (Orbit5k, $K=200$ unless swept, \texttt{cert\_1nn.py}).
\emph{Left:} the shortfall $\eps/(\rhonu/2)$ of Corollary~\ref{cor:nn}'s acceptance test against the landmark budget at $n=500$---flat near $12\times$ at the $\alpha=1.75$ default, where $\eps$ and $\rhonu$ both scale as $K^{-1/2}$, and rising from $9\times$ to $25\times$ at the capped $\alpha=4$ headline; the dashed line is the certified regime.
\emph{Right:} the median nearest-training-neighbor distance $\eps$ against training-set size, stationary over an eightfold range, with the acceptance threshold $\rhonu/2$ (equal for the two $\alpha$ to within $1\%$) $12$--$25\times$ below it.
The floor's tightness and the $\ell^\infty$ reading are reported in the text.}
\label{fig:obstruction}
\end{figure}

\paragraph{Local certificates, global accuracy.}
The conclusion we draw is not that the embedding is weak.
On the \emph{same} embedding a $q$-tuned RBF-SVM reaches $91.2\%$ where $1$-NN reaches $78\%$---the discriminative information is there, but it is \emph{global}, carried by the kernel across all training points, and not concentrated in local neighborhoods.
Every acceptance rule we have tested is local: $\eps$ to one neighbor, a ball, a per-class nearest distance.
Each inherits the same weak local geometry, which is why five independent levers fail at the same wall.

This also reframes the certified/accurate gap of Section~\ref{sec:cls} as something other than a deficiency of the nearest-centroid rule.
$\rhonu$ certifies a \emph{local} separation; the accuracy lives in \emph{global} structure.
The natural response is to attach the guarantee to the kernel head itself, and the two-sided bound makes that possible: each one-vs-one RBF decision function is Lipschitz in $\Psii$ with constant $L_{ab}=e^{-1/2}\sigma^{-1}\sum_i|\alpha_i^{ab}|$, exact from the duals, and Proposition~\ref{prop:upper} converts a margin in the embedding to a radius in $\dmatS$: if the predicted class wins every pairwise contest by $m_{co}$, the prediction is unchanged under any module perturbation smaller than $\min_o m_{co}/(L_{co}\sqrt S\Nmax)$.
We measured it (\texttt{cert\_svm.py}, $\alpha=4$, capped, $10$ folds, three seeds).
The head reaches $91.6\%$ and wins every contest on $99.8\%$ of test modules, so the certificate applies almost everywhere---and its median radius is $1.3\times10^{-6}$, three orders below $\rhonu$ and five below a single outlier's displacement; coverage is $0$ at every radius above $10^{-4}$.
The constants are the reason: $L_{ab}$ runs from $3\times10^2$ to $1.5\times10^4$, because the fitted head is a \emph{high-gain} function of the embedding---thousands of support vectors at $C=10$---while its margins are $O(1)$.
Even an impossibly tight $L=1$ would leave the radius near $7\times10^{-3}$, still below one outlier.
So the sixth lever closes with the others, and for the same reason from the other side: the kernel reads global structure through a sharp decision function, which is exactly what a Lipschitz certificate cannot cover.
A guarantee for the accurate head, if one exists, is not a Lipschitz one.
Nothing in the argument is particular to persistence: it applies to any landmark embedding whose lower gauge is witnessed by a single coordinate, and says that per-prediction certification by local acceptance requires local class separation that a globally-discriminative embedding need not provide.

We report this as measured on Orbit5k, not proved in general: a regime of lower intrinsic dimension or wider class separation could behave differently, and we would expect it to.
But within the geometry this paper works in, the shortfall is structural---neither the landmark budget nor the sample size moves it---and the operational content of the certificate lies at the level of the \emph{embedding}, where its conclusion holds on $100\%$ of the pairs audited above, rather than at the level of individual predictions.
The certified/accurate gap therefore has two distinct causes, not one---the nearest-centroid rule the certificate \emph{cannot} reach (the convex-hull obstruction, quantified above), and the nearest-neighbor rule it \emph{can} reach but which is weak here.

One design consequence is worth recording, since it is free.
$\rhonu=\tfrac14\tau\wmin$ is maximized by concentrating the weight budget on the active set $\mathcal A_s(\tau)$: landmarks below the radius threshold cannot witness at scale $\tau$ regardless, so weight on them is wasted from the certificate's point of view.
With uniform weights on the $K_{\mathrm{act}}$ active landmarks the floor improves by $\sqrt{K/K_{\mathrm{act}}}$.
Measured at $K=200$: $K_{\mathrm{act}}=98$ at $\tau=0.226$ ($1.4\times$) and $K_{\mathrm{act}}=29$ at $\tau=0.491$ ($2.6\times$).
The uniform choice $w_k=K^{-1/2}$ that the pipeline uses is thus the \emph{worst} case for the certificate, and a certificate-first configuration would spend its weight differently---at some cost in accuracy, since the inactive landmarks are not useless to the classifier.

\paragraph{Certified selective classification.}
Figure~\ref{fig:cert} plots selective accuracy of the certified nearest-centroid head against coverage, ranking test points by their certificate margin.
The margin is discriminative: accuracy climbs monotonically from $73.3\%$ at full coverage (this curve is a single seed; the $3$-seed mean is $73.4$) to $91.6\%$ on the most-certain tenth, and the certified minority reaches the \emph{uncertified} RBF-SVM accuracy of the same run ($90.8\%$; $91.2\%$ over three seeds) at about $24\%$ coverage---i.e., the top-ranked subset is as accurate as the best uncertified classifier.
Two caveats keep this honest.
The ranking statistic is the empirical nearest-centroid margin, a proxy for the $\rhonu$ floor rather than the floor itself, so this is selective classification, not certification; under the strict threshold of Theorem~\ref{thm:certified} only $1.5\%$ of points certify (at $92.2\%$ accuracy), and the strict certificate is not operational at this sample size.
And Lemma~\ref{lem:rbf} bounds the kernel separation of coherent pairs, not the fitted SVM's margin, so the $91.2\%$ head is not itself certified: the per-prediction guarantee of Theorem~\ref{thm:certified} attaches to the nearest-centroid head, at $73.4\%$.
Narrowing that gap is the main open problem this paper leaves; the paragraph below reports how far a per-prediction reading of the same theorem goes toward it.

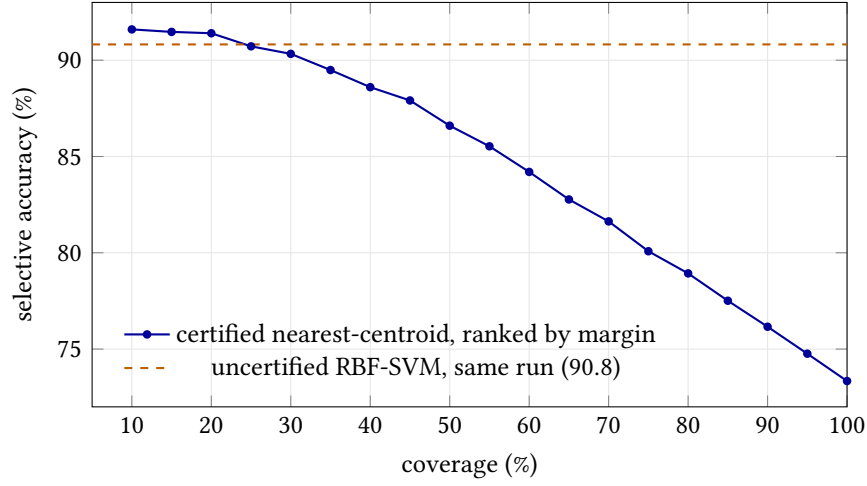
\begin{figure}[t]
\centering
\begin{tikzpicture}
\begin{axis}[width=0.7\linewidth, height=0.42\linewidth,
  xlabel={coverage (\%)}, ylabel={selective accuracy (\%)},
  xmin=5, xmax=100, ymin=72, ymax=93, grid=major, grid style={gray!20},
  tick label style={font=\small}, label style={font=\small},
  legend style={font=\small, draw=none, fill=none, at={(0.03,0.04)}, anchor=south west}]
\addplot[thick, color=blue!60!black, mark=*, mark size=1.2pt] coordinates {(10,91.60)(15,91.47)(20,91.40)(25,90.72)(30,90.33)(35,89.49)(40,88.60)(45,87.91)(50,86.60)(55,85.53)(60,84.20)(65,82.77)(70,81.63)(75,80.08)(80,78.93)(85,77.51)(90,76.16)(95,74.76)(100,73.34)};
\addlegendentry{certified nearest-centroid, ranked by margin}
\addplot[dashed, thick, color=orange!75!black] coordinates {(5,90.82)(100,90.82)};
\addlegendentry{uncertified RBF-SVM, same run ($90.8$)}
\end{axis}
\end{tikzpicture}
\caption{Certified selective classification (Orbit5k, near-diagonal fan $S=8$).
Selective accuracy of the certified nearest-centroid head vs.\ coverage; dashed line is the uncertified RBF-SVM.}
\label{fig:cert}
\end{figure}

\paragraph{The certificate, read per prediction.}
Section~\ref{sec:cls} noted that the radius of Theorem~\ref{thm:certified} is a norm bound while the quantity a single prediction needs is a scalar; Table~\ref{tab:pointwise} measures what that distinction is worth, on the eleven benchmarks of Paper~II's firing study (\texttt{pointwise\_cert.py}, five seeds $\times$ ten folds).
Read per point, the paper's own radius \eqref{eq:rm} certifies $8.3\%$ of Orbit5k test modules at $84.7\%$ accuracy against $63.3\%$ overall, and between $32\%$ and $54\%$ on the four graph sets where it fires at all.
The non-asymptotic per-point form is stronger wherever the classes are well separated---$74.0\%$ of NCI1, $71.4\%$ of DD, $68.7\%$ of NCI109---and it is on Orbit5k, where it triples coverage ($26.9\%$ against $8.3\%$) at the same accuracy, that the mechanism is clearest: the norm-based radius pays for the ambient dimension, and a single prediction does not have to.
The Gaussian form goes further still, at a confidence reduced by the Berry--Esseen charge in the last column.
The pattern where nothing certifies is the one a certificate should have: on COX2, DHFR, PTC\_MR and IMDB-M\textsc{ulti} the nearest-centroid head is at or below the constant classifier, and neither theorem-backed form certifies a single point there, the Gaussian form reaching only $7.8\%$ on DHFR and $0.1\%$ on IMDB-M\textsc{ulti}.
IMDB-B\textsc{inary} is the one omitted set where the head is well above its baseline ($62.6$ against $50.0$) and the two theorem-backed forms are still silent, while the Gaussian form certifies $45.8\%$ at $73.6\%$---the clearest case that the Berry--Esseen charge, not the geometry, is what stands between the measured coverage and a guarantee.
This does not close the gap of the preceding paragraph---the head being certified is still the nearest-centroid one, which trails the kernel head by $18$ points at the headline configuration---but it moves the per-prediction guarantee from ``operationally empty'' to ``covers most predictions on three of eleven benchmarks'', which is a different sentence about the same theorem.

\begin{table}[t]
\centering
\caption{The nearest-centroid certificate read per prediction (five seeds $\times$ ten folds; \texttt{pointwise\_cert.py}).
Cells give the percentage of test points certified and, in parentheses, nearest-centroid accuracy on them; ``NC'' is that head's accuracy over all test points.
``Radius \eqref{eq:rm}'' is the gap test of Theorem~\ref{thm:certified} evaluated per point, which is valid on the concentration event whether or not the rule-level condition $r_m<\tfrac12\widehat\Delta$ holds; ``Bernstein'' and ``Gaussian'' are the two forms of Proposition~\ref{prop:pointwise}; ``BE'' is the Berry--Esseen charge subtracted from the Gaussian form's confidence.
Orbit5k is the capped single-bifiltration configuration at $\alpha=4$; the graph rows are the multi configuration of Table~\ref{tab:graph}.
The five omitted benchmarks certify nothing under the radius and Bernstein columns.}
\label{tab:pointwise}
\small
\begin{tabular}{lccccc}
\toprule
Data set & NC & Radius \eqref{eq:rm} & Bernstein & Gaussian & BE \\
\midrule
Orbit5k  & $63.3$ & $8.3$ ($84.7$)  & $26.9$ ($85.9$) & $67.0$ ($73.2$) & $0.28$ \\
MUTAG    & $87.0$ & $0$             & $0$             & $66.1$ ($91.8$) & $0.33$ \\
PROTEINS & $58.5$ & $31.6$ ($66.6$) & $9.5$ ($90.3$)  & $84.1$ ($61.0$) & $0.26$ \\
NCI1     & $63.6$ & $54.0$ ($67.8$) & $\mathbf{74.0}$ ($66.4$) & $88.7$ ($65.2$) & $0.10$ \\
NCI109   & $63.6$ & $47.5$ ($68.2$) & $\mathbf{68.7}$ ($67.1$) & $85.8$ ($65.3$) & $0.09$ \\
DD       & $74.8$ & $45.1$ ($80.0$) & $\mathbf{71.4}$ ($79.2$) & $92.0$ ($76.4$) & $0.13$ \\
\bottomrule
\end{tabular}
\end{table}

\section{Adapting the filtration: a certified spectrum}\label{sec:dcomplex}

Everything so far fixes the bifiltration.
We now relax that along a spectrum of adaptivity---\emph{fixed} $\to$ \emph{selected} $\to$ \emph{learned}---every step of which keeps the certificate: the $\rhonu$ floor of Theorem~\ref{thm:cert} constrains the \emph{output} embedding, so it holds however the bifiltration is produced.
Two adaptive regimes result.
\textbf{SELECT-COMPLEX} chooses the two filtration axes in closed form, from a bank of candidate node functions, by a train-fold criterion---the PALACE principle (adaptive by \emph{selection}, not descent) applied one level up, to the filtration.
\textbf{D-COMPLEX} learns them by gradient, in the manner of D-\textsc{Gril}~\citep{Mukherjee2024}, but retaining the per-prediction guarantee a learned pipeline otherwise forfeits.
The experiments return the recurring lesson: the fixed default is robust, adaptation buys only modest, dataset-dependent gains, and closed-form selection is the more stable way to buy them.

\subsection*{SELECT-COMPLEX: closed-form filtration selection}
PALACE is adaptive not by training but by \emph{selection}---it chooses landmarks in closed form by a data criterion---and the same principle chooses the filtration.
From a bank of candidate node functions (heat-kernel signatures at several diffusion times, degree, closeness and eigenvector centrality, clustering), SELECT-COMPLEX scores each candidate pair on the training fold, keeps the certifiable ones, and freezes the best pair.
There is no gradient and no held-out split: selection is train-only, so the certificate and the fold protocol are untouched.

Following PALACE, we rank candidate pairs by a shrinkage-regularized Fisher--Mahalanobis margin under the pooled within-class covariance,
\begin{equation}\label{eq:maha}
  J(\LC)\ =\ \min_{c\ne c'}\ (\widehat\mu_c-\widehat\mu_{c'})^{\!\top}
             \bigl(\widehat S_W+\lambda I\bigr)^{-1}(\widehat\mu_c-\widehat\mu_{c'}),
\end{equation}
with a diagonal $\widehat S_W$ (the full $K\times K$ covariance is singular at these sample sizes) and $\lambda=\overline{\mathrm{var}}$.
Appendix~\ref{app:select} records why the simpler ratio of centroid gap to mean within-class spread fails---it prefers high-variance nuisance directions and, on MUTAG, selects a pair that transfers $3.7$ points worse---and why the shrinkage level matters ($-13.9$ at $\lambda/10$).

Table~\ref{tab:select} reports selected against fixed at equal net and pipeline, under both the closed-form criterion and an inner-cross-validated one.
Selection pays only where the fixed pair is genuinely suboptimal---$+1.9$ on COX2 (closed-form; $+2.8$ under inner-CV), against $-2.1$ on MUTAG and $-1.0$ on the label-heavy DHFR---and the ranking of the two criteria is the same on all three sets.
Inner-CV is $0.7$ points better on average, and we nonetheless report the closed-form criterion as the default: it keeps the selector genuinely closed-form, as Section~\ref{sec:adaptive} claims of the whole construction, and it is far more stable, agreeing with itself on $5/5$ folds on COX2 where inner-CV splits over three pairs and, on DHFR, chooses five different pairs in five folds.
We report both so the cost of that choice is visible.
Once essential classes are capped, the fixed HKS$\times$degree baseline is itself much stronger on MUTAG ($85.6\to88.8$), and the headroom selection used to occupy is gone.
On DHFR the selection is also unstable, choosing a different pair in every one of the five folds and inflating the fold standard deviation to $\pm7.4$.
The value is therefore not an accuracy jump---on two of three sets there is none---but a \emph{training-free, certified} route to a data-adaptive filtration: no gradient, no warm start, and more stable than the differentiable version below.

\paragraph{SELECT-MULTI: selecting the set, not the pair.}
Single-pair selection asks which \emph{one} bifiltration to keep; the stronger question is which $k$ to \emph{concatenate}.
SELECT-MULTI scores all $\binom72=21$ candidate pairs from the node-function bank on the training fold (3-fold inner nearest-centroid accuracy on a cheap $K=40$ stack), keeps the top $k$, concatenates them at the same total budget as the fixed multi configuration, and freezes---per fold, train-only, so the protocol and the certificate are untouched.
The gains are monotone in $k$ on four of six sets, and at $k=5$ (the ``sel.''\ column of Table~\ref{tab:graph}) the comparison against the fixed three-pair bank reads $+3.7$ on NCI1, $+1.7$ on DHFR, $+0.3$ on MUTAG, $-0.2$ on COX2, $-1.3$ on PROTEINS and $-2.9$ on PTC\_MR.
The pattern is the one selection should produce: it pays precisely where the hand-picked bank omits the discriminating axes---on NCI1 the selector concentrates on an eigenvector-centrality family (eigen$\times$label plus HKS$\times$eigen, stable across folds and seeds) that the fixed bank does not contain---and it costs a little where the fixed bank was already right.
This revises the verdict of Table~\ref{tab:select}: closed-form selection is unreliable at the single-pair level, and useful at the set level.
(On IMDB-B the bank harness lacks the essential-class cap that matters most there, and every configuration including selection reads $64$--$66$ against the capped single bifiltration's $70.0$; we report the capped number in Table~\ref{tab:graph} and flag the harness gap rather than mix the two.)

\begin{table}[t]
\centering
\caption{SELECT-COMPLEX: selected vs.\ fixed bifiltration ($5$-fold CV accuracy $\%$, RBF-SVM), under two selection criteria.
Both are train-only, so the certificate and the fold protocol hold for either.
The \emph{closed-form} criterion is the shrinkage-regularized Fisher/Mahalanobis margin \eqref{eq:maha}; \emph{inner-CV} is a $3$-fold cross-validated nearest-centroid accuracy on the training fold.
Inner-CV is $0.7$ points better on average but is a held-out search, so we report the closed-form criterion as the default and inner-CV as the variant.
``pair (folds)'' gives the modal choice and how many of the five folds agree on it---the closed-form criterion is markedly more stable.
This experiment uses a lighter configuration than Table~\ref{tab:graph} ($5$ folds, $K=120$ landmarks, $S\le6$ slices), so only the within-row comparison is meaningful.}
\label{tab:select}
\begin{tabular}{lccccc}
\toprule
 & & \multicolumn{2}{c}{closed-form (Mahalanobis)} & \multicolumn{2}{c}{inner-CV} \\
\cmidrule(lr){3-4}\cmidrule(lr){5-6}
Dataset & fixed (HKS$\times$deg) & selected & pair (folds) & selected & pair (folds) \\
\midrule
COX2  & $78.2\pm1.6$ & $\mathbf{80.1\pm3.0}$ & HKS$_1\times$clus.\ ($5/5$)
      & $81.0\pm3.8$ & HKS$_1\times$clos.\ ($3/5$) \\
MUTAG & $\mathbf{88.8\pm1.0}$ & $86.7\pm4.0$ & HKS$_{10}\times$deg.\ ($4/5$)
      & $88.3\pm2.7$ & HKS$_{10}\times$clos.\ ($2/5$) \\
DHFR  & $\mathbf{76.9\pm1.3}$ & $75.9\pm3.2$ & HKS$_1\times$clus.\ ($2/5$)
      & $75.6\pm7.4$ & (five pairs in five folds) \\
\bottomrule
\end{tabular}
\end{table}

\subsection*{D-COMPLEX: differentiable filtration}
\paragraph{Learned bifiltration.}
For a graph or point cloud with node/point features $X$, a network emits two positive grades per node, $\gtheta(v)=\mathrm{softplus}(\ftheta(X_v))\in\R^2_{>0}$, in place of the hand-picked geometric scale and codensity; the fixed bifiltration is the constant special case.
The net~\eqref{eq:net} stays frozen: Section~\ref{sec:exp} showed that learning the slice directions overfits, so learning is spent only on the filtration, where the data---node labels on molecules, density on clouds---actually varies.

\paragraph{Differentiable slice-stack.}
Along slice $s$ a node appears at $u_v^{(s)}=\max(\gtheta(v)_1/\cos\vartheta_s,\gtheta(v)_2/\sin\vartheta_s)$ and a flag simplex at $s_\sigma^{(s)}=\max_{v\in\sigma}u_v^{(s)}$ (the net is the fixed fan, so $\beta_s=0$); both grades are monotone under inclusion, so $s^{(s)}$ is a valid filtration for every $\gtheta$.

\begin{proposition}[Gather gradient]\label{prop:grad}
Fix $\theta_0$ such that (a) the persistence pairing is constant on a neighborhood of $\theta_0$ and (b) for every simplex $\sigma$ appearing as a birth or death simplex, the maximizing vertex $\argmax_{v\in\sigma}u^{(s)}_v$ is unique.
Then on that neighborhood each finite bar has $b_i=s^{(s)}_{\sigma_b}$, $d_i=s^{(s)}_{\sigma_d}$, and
\[
  \frac{\partial b_i}{\partial\gtheta(v)}
  =\mathbf 1\bigl[v=\argmax_{u\in\sigma_b}u^{(s)}_u\bigr]\,
   \frac{\partial u^{(s)}_v}{\partial\gtheta(v)},
\]
the inner derivative being $1/\cos\vartheta_s$ or $1/\sin\vartheta_s$ on whichever grade attains the max in $u^{(s)}_v$ (and likewise for $d_i$).
Consequently any barcode vectorization differentiable in the bar coordinates---in particular the PLACE/PALACE coordinates~\eqref{eq:coord}---is differentiable in $\gtheta$ at $\theta_0$.
\end{proposition}
\noindent\emph{Proof.} Deferred to Appendix~\ref{app:proofs}.
\takeaway{proved}{Bars of the learned bifiltration are differentiable in the network output wherever pairings and argmax vertices are locally constant.}
Condition (b) is not generic---on labeled molecular graphs ties in the max-of-vertices filtration are structural, and there the gather rule returns a Clarke subgradient; the implementation and the precise sense in which D-COMPLEX is differentiable are in Appendix~\ref{app:clarke}.

\paragraph{Certified training.}
The certificate of Theorem~\ref{thm:cert} constrains the geometry of the \emph{output} embedding, not the origin of the bifiltration: if $\norm{\Psii_\theta(M)-\Psii_\theta(M')}\ge\rhonu(\tau)$ holds for the trained $\ftheta$, the nearest-centroid guarantee (Theorem~\ref{thm:certified}) and the RBF-margin corollary (Lemma~\ref{lem:rbf}) apply verbatim.
The conclusion is indifferent to how the bifiltration was produced---but the \emph{hypotheses} are not: $\tau$-admissibility of the frozen landmark configurations and coherence are both functions of $\gtheta$, and the cold-start failure of Appendix~\ref{app:dcomplex} shows training can destroy admissibility outright.
So the certificate must be re-audited after training, and we did (\texttt{dcomplex\_audit.py}: MUTAG, DHFR and PROTEINS, fold $0$, the same sampled pairs, frozen landmarks and $\tau$ for both filtrations).
Coherence and the $\rhonu$ conclusion hold on $100\%$ of audited pairs for the fixed filtration and again for the learned one on all three sets; what training changes is the scale, not the guarantee---the median sliced matching distance moves to $0.84\times$, $1.29\times$ and $1.11\times$ of its fixed value, with no consistent direction---so the certificate survives at whatever scale results.
D-COMPLEX makes the floor an objective, training
\[
  \mathcal L(\theta)=\mathcal L_{\mathrm{CE}}(\theta)
    +\lambda\,\mathbb E\big[\mathrm{relu}\big(1-(z_y-\textstyle\max_{c\ne y}z_c)\big)\big],
\]
whose margin term is a differentiable surrogate for the embedded separation $\Delta$, and hence for the certified floor.
A purely accuracy-trained model such as D-\textsc{Gril} cannot express this term: it has no per-prediction floor to enlarge.

\paragraph{Gradient learning: certified, but not safer.}
Gradient learning of the filtration needs two conditions or it fails outright---the PLACE/PALACE head and a warm start at the fixed filtration (Appendix~\ref{app:dcomplex}).
With both, Table~\ref{tab:dcomplex} gives the verdict over three seeds: the learned filtration helps on three of five sets and hurts on two, the closed-form single-pair selector helps on one of three, and neither reliably beats the fixed hand-picked bifiltration.
The adaptation that does pay is selecting the \emph{set} (SELECT-MULTI above), which is closed-form and certificate-preserving.
We no longer state this as ``no knob moves'': the kernel head (Table~\ref{tab:kernelsel}), the radius factor (Remark~\ref{rem:radcap}) and the bifiltration set each move accuracy by measurable amounts, all in closed form; it is gradient-shaped adaptation specifically that buys nothing.
The differentiable route therefore earns its place as the way to \emph{certify a trained model}---the cell D-\textsc{Gril} leaves empty, and the audit above shows the certificate does survive training---and not as a path to accuracy.

\begin{table}[t]
\centering
\caption{Adapting the filtration: gain over the fixed HKS$\times$degree bifiltration (accuracy points, mean over $3$ seeds, fixed net).
Neither route is reliable: closed-form selection helps on COX2 and hurts on MUTAG and DHFR; gradient learning helps on three of five sets and hurts on two.
The selector column is the fold-matched difference of Table~\ref{tab:select} under the closed-form criterion and is reported only where that experiment was run.}
\label{tab:dcomplex}
\begin{tabular}{lcc}
\toprule
Dataset & selected $-$ fixed & learned $-$ fixed \\
\midrule
COX2     & $+1.9$ & $+0.8$ \\
MUTAG    & $-2.1$ & $-2.0$ \\
DHFR     & $-1.0$ & $-2.0$ \\
NCI1     & --     & $+3.4$ \\
PROTEINS & --     & $+1.0$ \\
\bottomrule
\end{tabular}
\end{table}

\section{Discussion}\label{sec:disc}

COMPLEX shows that the closed-form, certified single-parameter kernel of PLACE/PALACE lifts to multiparameter modules without leaving closed form and without conceding accuracy: on Orbit5k it matches or overtakes every baseline it competes with, Euler-characteristic surfaces included, and at the $100{,}000$-cloud scale it finishes ahead of the transformers and the GNN-trained graphcode alike~\citep{KerberRussold2024}.
Section~\ref{sec:dcomplex} shows the same certificate survives \emph{adapting} the bifiltration---by closed-form selection or by gradient---and that such adaptation buys modest, dataset-dependent gains over the already-strong fixed default.
We close with the trade-offs, stated as trade-offs.

\paragraph{Sliced versus true matching distance.}
A finite net certifies the sliced surrogate $\dmatS\le\dmat$, not the matching distance itself, whose exact computation is polynomial but impractical~\citep{KerberLesnickOudot2019}.
The witnessing slice makes the \emph{comparison} pair-adaptive---each pair is certified on the slice that separates it most---while the net stays fixed and closed-form, so the $\sqrt S$ Lipschitz constant and the linear-time admissibility check are the price of a deterministic, tuning-free guarantee.

\paragraph{Closed-form versus gradient-trained features.}
The comparison with D-\textsc{Gril}~\citep{Mukherjee2024} was categorical---it learns the bifiltration and cannot certify---and D-COMPLEX (Section~\ref{sec:dcomplex}) removes the excuse for choosing between the two: it learns the bifiltration \emph{and} certifies.
We had read the learned filtration as helping where node features carry signal beyond geometry, and that reading does not survive measurement: it helps on NCI1 ($+3.4$), PROTEINS ($+1.0$) and COX2 ($+0.8$) but hurts on MUTAG ($-2.0$) and DHFR ($-2.0$), and MUTAG is a labeled molecular set like the ones it helps on.
We have no rule that predicts which case a dataset falls into, and we no longer claim one; with the closed-form single-pair selector also negative on two of the three sets where it was run, the fair summary is that adapting a \emph{single} bifiltration is not a dependable lever in either regime---selecting the bifiltration \emph{set} (SELECT-MULTI, Section~\ref{sec:dcomplex}) is the adaptation that pays, and it pays by concatenation, not replacement.
The certified/accurate gap---$73.4$ against $91.2$ from the same run---remains the honest limitation, and Section~\ref{sec:exp} now says what shape it has: a local certificate set against a global head, an obstruction measured from both sides and closed for every Lipschitz route.

\paragraph{Limits.}
The constants depend on the per-slice bar budget $N$; this is intrinsic, not an artifact, since persistence metrics admit no bi-Lipschitz embedding into Hilbert space~\citep{Carriere2019-nx,Zava2025} and the unbounded module space no quantitative one.
The certified nearest-centroid \emph{floor} is conservative---few pairs clear the strict threshold---so the operational statement is the risk-coverage curve, not the strict certified fraction.
D-COMPLEX is deliberately minimal: it learns only the node grades, requires a warm start at the closed-form solution, and freezes the landmarks; jointly learning the landmarks, replacing the grade network with a full message-passing filtration, and scaling to the larger bio-activity benchmarks where D-\textsc{Gril} is strongest are the natural next steps, each of which the gather gradient (Proposition~\ref{prop:grad}) and the certificate-preserving loss already accommodate.
The scale limit is measured rather than assumed: on \texttt{ogbg-molhiv} the fixed embedding trails trained graph networks by three to four ROC-AUC points (Section~\ref{sec:exp}), so the parity this paper demonstrates is parity on point clouds and the TU benchmarks, and closing the OGB gap---most plausibly with the learned filtration of Section~\ref{sec:dcomplex} rather than a fixed one---is the clearest open experimental problem it leaves.

\appendix
\section{Deferred proofs}\label{app:proofs}

\subsection*{Proof of Proposition~\ref{prop:upper}}
\begin{proof}
Fix a slice, write $A=M|_{\ell_s}$, $B=M'|_{\ell_s}$ and $\delta=\db(A,B)$, represent both as $\Delta$-padded $N$-tuples $[x_1,\dots,x_N]$, $[y_1,\dots,y_N]$, and let $\psi\in S_N$ realize $\delta$ in \eqref{eq:db}.
For a landmark $(p,r)$ obeying \eqref{eq:radcap}, $\varphi_{p,r}$ is $1$-Lipschitz for $\db$ on $\D{1}$ and vanishes at $\Delta$, so
\[
  |\varphi_{p,r}(A)-\varphi_{p,r}(B)|
  \ \le\ \sum_{i=1}^{N}\bigl|\varphi_{p,r}(x_i)-\varphi_{p,r}(y_{\psi(i)})\bigr|
  \ \le\ \sum_{i=1}^{N}\db\bigl(x_i,y_{\psi(i)}\bigr)\ \le\ \Nmax\,\delta ,
\]
the padding entries included (the $\Delta$-terms contribute $|\varphi_{p,r}(x_i)|\le\db(x_i,\Delta)\le\delta$).
Squaring, summing over landmarks and using $\sum_kw_k^2=1$ gives $\norm{\Phii(A;\LC_s)-\Phii(B;\LC_s)}\le\Nmax\,\db(A,B)$.
Finally, by \eqref{eq:embed}, $\norm{\Psii(M)-\Psii(M')}^2=\sum_s\omega_s^2\norm{\Phii(A_s)-\Phii(B_s)}^2 \le\Nmax^2\sum_s\bigl(\omega_s\db(A_s,B_s)\bigr)^2\le S\,\Nmax^2\,\dmatS(M,M')^2$.
\end{proof}

\subsection*{Proof of Proposition~\ref{prop:suff}}
\begin{proof}
Let $s^\star$ be the witnessing slice, $A=M|_{\ell_{s^\star}}$, $B=M'|_{\ell_{s^\star}}$, $\delta=\db(A,B)$ and $t^\star=\tau/\omega_{s^\star}$.
Since $\omega_{s^\star}\delta=\dmatS(M,M')\in[\tau,D]$ and $\omega_{s^\star}\ge\omega_{\min}$, we have $t^\star\le\delta\le D/\omega_{\min}$.
Two distinct points of $A$ are bars that differ in a birth or a death, i.e.\ in the entry parameter of a generator or a relation, so by (i) they are more than $4D/\omega_{\min}\ge4\delta$ apart in $\norm\cdot_\infty$.
By (ii) every point of $A\cup B$ has $\db(\cdot,\Delta)>D/\omega_{\min}\ge\delta$, so a matching $\psi$ realizing $\delta$ in \eqref{eq:db} pairs no point with a padding $\Delta$ (that would cost more than $\delta$) and compares every matched pair directly, the diagonal route in \eqref{eq:db1} costing more than $\delta$: $|A|=|B|$, every $a\in A$ is paired with $\psi(a)\in B$ at $\norm{a-\psi(a)}_\infty\le\delta$, and some pair $(a^\star,b^\star=\psi(a^\star))$ attains $\norm{a^\star-b^\star}_\infty=\delta$.

Admissibility of $\LC_{s^\star}$ at the realized scale gives $\lambda_0:=\lambda_0(\LC_{s^\star})\ge t^\star/4$ and $\max_kr_k\le(t^\star+\lambda_0)/2$; since $\lambda_0\le\max_kr_k$ always, this forces $\max_kr_k\le t^\star$.
By the definition \eqref{eq:lebesgue} of the functional Lebesgue number there is $k^\star$ with $\varphi_{k^\star}(a^\star)\ge\lambda_0$, i.e.\ $\norm{p-a^\star}_\infty\le r-\lambda_0$ for $p=p_{k^\star}$, $r=r_{k^\star}$; as $r\ge\varphi_{k^\star}(a^\star)\ge\lambda_0\ge t^\star/4$, $k^\star\in\mathcal A_{s^\star}(\tau)$.
Three facts locate the mass in the cone of $k^\star$.
(1) For $a'\in A\setminus\{a^\star\}$: $\norm{a'-p}_\infty\ge\norm{a'-a^\star}_\infty-\norm{a^\star-p}_\infty >4\delta-r\ge3\delta\ge r$, so $\varphi_{k^\star}(a')=0$.
(2) $\norm{b^\star-p}_\infty\ge\norm{a^\star-b^\star}_\infty-\norm{a^\star-p}_\infty\ge\delta-r+\lambda_0 \ge t^\star-\tfrac12(t^\star+\lambda_0)+\lambda_0=\tfrac12(t^\star+\lambda_0)\ge r$, so $\varphi_{k^\star}(b^\star)=0$.
(3) For $b\in B\setminus\{b^\star\}$, paired with some $a=\psi^{-1}(b)\ne a^\star$: if $\norm{b-p}_\infty<r$ then $\norm{a-a^\star}_\infty\le\norm{a-b}_\infty+\norm{b-p}_\infty+\norm{p-a^\star}_\infty <\delta+r+r\le3\delta$, contradicting the separation in (1); so $\varphi_{k^\star}(b)=0$.
Hence $\varphi_{k^\star}(A)=\varphi_{k^\star}(a^\star)\ge\lambda_0\ge t^\star/4$ and $\varphi_{k^\star}(B)=0$, so $|\Phii_{k^\star}(A)-\Phii_{k^\star}(B)|=w_{k^\star}\varphi_{k^\star}(a^\star)\ge w_{k^\star}t^\star/4$, which is Definition~\ref{def:coh}.
Stating (i) along the slices is essential: an $\ell^\infty$ separation of the grades in $\R^d$ does \emph{not} imply separation after restriction, since a slice contracts grade space---at $\vartheta=\pi/4$ the grades $(0,t)$ and $(t,0)$ are $\ell^\infty$-far but enter at the same parameter.
Hypothesis (ii) is what rules out short bars of $M'$ inside the witnessing cone, which would otherwise be free to cancel the witness; without it the conclusion can fail.
\end{proof}

\subsection*{Proof of Theorem~\ref{thm:certified}}
\begin{proof}
Condition on the class labels, so that the class sizes $m_c$ are fixed and the $m_c$ embeddings $\Psii(M_i)$ with $y_i=c$ are i.i.d.\ from the class-conditional law; the bound proved under this conditioning holds unconditionally (classes with $m_c=0$ have no empirical mean and are excluded from the maximum).
$\Psii(M)-\mu_c$ is a centered $\R^{SK}$-valued variable with $\norm{\Psii(M)-\mu_c}\le2B$ a.s.\ (Lemma~\ref{lem:bdd}), so Pinelis's Bernstein inequality in a $2$-smooth Banach space~\citep{Pinelis1994}, applied in the Hilbert space $\R^{SK}$, gives $\Pr[\norm{\widehat\mu_c-\mu_c}>r]\le2\exp\bigl(-\tfrac{m_cr^2}{2(V_c+2Br/3)}\bigr)$.
Setting the right side to $\delta/C$ and solving the resulting quadratic, then using $\sqrt{a+b}\le\sqrt a+\sqrt b$, yields the sufficient radius \eqref{eq:rm}; a union bound over the $C$ classes gives the coverage event.
On that event $\bigl|\norm{x-\widehat\mu_c}-\norm{x-\mu_c}\bigr|\le\norm{\widehat\mu_c-\mu_c}\le r_m$ for every $x$ and $c$.
Hence if $\mathrm{gap}(M)>2r_m$, writing $\widehat c=\widehat c(M)$ and $x=\Psii(M)$, for every $c\ne\widehat c$
\[
  \norm{x-\mu_{\widehat c}}\ \le\ \norm{x-\widehat\mu_{\widehat c}}+r_m
  \ <\ \norm{x-\widehat\mu_c}-2r_m+r_m\ \le\ \norm{x-\mu_c},
\]
so the population rule also selects $\widehat c$.
Finally $\norm{\mu_c-\mu_{c'}}\ge\widehat\Delta-2r_m>0$ under $r_m<\tfrac12\widehat\Delta$.
\end{proof}

\subsection*{Proof of Proposition~\ref{prop:grad}}
\begin{proof}
Fix $\theta_0$ satisfying (a) and (b) and a neighborhood $U$ of $\theta_0$ on which the pairing is constant.
On $U$ every finite bar $i$ has a fixed birth simplex $\sigma_b$ and death simplex $\sigma_d$, so $b_i(\theta)=s^{(s)}_{\sigma_b}(\theta)$ and $d_i(\theta)=s^{(s)}_{\sigma_d}(\theta)$ with $s^{(s)}_\sigma=\max_{v\in\sigma}u^{(s)}_v$.
By (b) the maximizing vertex $v_b=\argmax_{v\in\sigma_b}u^{(s)}_v$ is unique at $\theta_0$; the finitely many maps $\theta\mapsto u^{(s)}_v(\theta)$ are continuous, so $v_b$ remains the unique maximizer on a smaller neighborhood, on which $s^{(s)}_{\sigma_b}=u^{(s)}_{v_b}$ and hence $\partial b_i/\partial\gtheta(v)=\mathbf 1[v=v_b]\,\partial u^{(s)}_{v_b}/\partial\gtheta(v_b)$.
The inner map $u^{(s)}_v=\max\bigl(\gtheta(v)_1/\cos\vartheta_s,\ \gtheta(v)_2/\sin\vartheta_s\bigr)$ is the maximum of two affine functions of $\gtheta(v)$; wherever the attaining grade is unique it coincides locally with that branch, with derivative $1/\cos\vartheta_s$ in the first grade and $0$ in the second, or $1/\sin\vartheta_s$ and $0$ respectively, which is the displayed inner derivative; at a tie between the two grades the map is locally Lipschitz and the gather rule returns one element of its Clarke subdifferential (Remark~\ref{rem:clarke}).
The same argument applies to $d_i$.
Finally, a vectorization differentiable in the bar coordinates composes with $(b_i,d_i)$ by the chain rule; the hat coordinates \eqref{eq:coord} are piecewise linear in the bar coordinates and differentiable off the finitely many hyperplanes where $\norm{p_k-(b_i,d_i)}_\infty=r_k$ or where the two coordinates of that $\ell^\infty$ norm tie, so $\Psii_\theta(M)$ is differentiable in $\theta$ at every $\theta_0$ satisfying (a), (b) and lying off those hyperplanes.
\end{proof}

\section{Experimental details}\label{app:details}

\subsection*{The certificate-1NN sweep: budget, sampling density, metric, and the headline configuration}\label{app:sweep}
The paragraphs below continue the study of Section~\ref{sec:exp} at the $\alpha=1.75$ default and then repeat it at the headline configuration; the numbers they establish are the ones summarized under ``The obstruction, stated.''

Nor is the landmark budget a lever, in either direction.
Writing the shortfall as the ratio $\eps/(\rhonu/2)$, both quantities scale as $K^{-1/2}$---$\rhonu$ by construction and $\eps$ because $\wmin=K^{-1/2}$ makes the embedded norm an average over landmarks---so the ratio is invariant.
Measured across $K\in\{8,16,32,64,128,200\}$ it moves from $12.0$ to $12.9$ at $\tau=0.491$, and from $26.1$ to $28.1$ at $\tau=0.226$: flat to within the sampling noise over a $25\times$ range of budgets.
There is therefore no small-$K$ regime in which the certificate becomes operational, contrary to what the $K^{-1/2}$ dependence of $\rhonu$ alone would suggest.

Sampling density barely moves it either.
Over an eightfold range of training set size ($500$ to $4{,}000$ modules) the median nearest-neighbor distance $\eps$ is stationary---$0.056$, $0.053$, $0.055$, $0.051$---because in a $1{,}600$-dimensional embedding those distances concentrate, so more data does not bring a test module much closer to its own class.
The scale $\tau$ meanwhile \emph{falls} ($0.491$, $0.468$, $0.417$, $0.383$), being a quantile of the training cross-class separation, which denser sampling can only tighten; at that outermost $\tau$ the shortfall therefore widens rather than closes, from $12.9$ to $15.1$ (from $4.9$ to $8.6$ measured against the active-set floor).
Coverage is exactly $0$ up to $n=2{,}000$ and reaches a single accepted test module in $400$ at $n=4{,}000$ ($0.3\%$, correctly classified).
We record that as the first nonzero cell rather than as a trend: it is one point.
What does improve with density is the head the certificate can reach---$1$-NN accuracy rises from $66.7\%$ to $78.0\%$ across the same range---so the obstruction is specific to the acceptance radius, not to the representation.
Larger $n$ was not testable here: the audit materializes an $n\times n\times 1000$ array, which needs $128$\,GB at $n=4{,}000$ and grows quadratically.

The metric is the last free choice, and it comes closest.
Coherence asserts only that \emph{some} landmark witnesses, so the $\ell^2$ bound already rests on a single coordinate and pays $\wmin=K^{-1/2}$ for the privilege.
Read the same conclusion in $\ell^\infty$ over the $\omega$-scaled coordinates and the floor is $\tau/4$ rather than $\tau/(4\sqrt K)$, larger by $\sqrt K$; what it costs is that $\eps$ becomes a maximum rather than a root-mean-square, so the shortfall improves by $\sqrt{KS}/(\max/\mathrm{RMS})$.
Measured at $K=200$, $\max/\mathrm{RMS}=9.4$, an improvement of $4.3\times$ against the $12\times$ required---the median test module sits $32\times$ above the $\ell^2$ threshold and $7.6\times$ above the $\ell^\infty$ one.
Nor does the improvement grow with the budget: $\max/\mathrm{RMS}$ itself rises roughly as $\sqrt K$ ($2.7$, $4.3$, $5.7$, $6.9$, $8.7$, $9.4$ for $K=8,\dots,200$), because $\Delta\varphi$ spreads over more landmarks as the budget grows, so the net gain returns only the $\sqrt S$ and is flat in $K$.
Not one additional test module crosses the threshold.

\paragraph{The same sweep at the headline configuration.}
Everything above was measured at the $\alpha=1.75$ default.
Re-running the full sweep on the certified capped construction at $\alpha=4$ ($3$ seeds, $K\in\{8,\dots,200\}$, $n\in\{500,\dots,4{,}000\}$) confirms the obstruction and sharpens it in one respect.
Coverage never exceeds $0.5\%$ at $K=200$ in any cell; the only cells above $1\%$ are at $K=8$ (up to $2.5\%$), where the accepted handful is classified well below the head's accuracy---acceptance there is noise, not certification.
The median nearest-neighbor distance is larger, $0.106$ at $n=500$ falling only to $0.088$ at $n=4{,}000$, so the shortfall at the outermost $\tau$ is $25\times$ rather than $12\times$, and flat in $n$ ($24.6$, $24.1$, $25.4$, $26.0$).
The one qualitative change is that the shortfall is no longer invariant in $K$: it grows from $9.0\times$ at $K=8$ to $24.6\times$ at $K=200$ (from $19\times$ to $53\times$ at the median $\tau$), because under the radius cap $\eps$ falls more slowly than $K^{-1/2}$ while $\rhonu$ still does.
That does not open a small-$K$ regime---the ratio never approaches $1$, and the $K=8$ coverage is the noise just described---but it means the landmark budget is a lever \emph{against} the certificate at the headline configuration, not a neutral one.
The nearest-neighbor geometry is unchanged: $d_{(2)}/d_{(1)}$ is $1.12$--$1.13$ at every $n$, $1$-NN accuracy $74$--$78\%$, the $\ell^\infty$ reading returns $\max/\mathrm{RMS}=7.8$ at $K=200$ (rising $2.8\to7.8$ across the budget, as before), and $5$--$7\%$ of test modules have a runner-up class even $1.5\times$ behind.
The conclusion is the one already drawn, now at the configuration the paper reports.

\subsection*{Why the simpler selection criterion fails}\label{app:select}
The score itself deserves a word, because a natural choice fails.
Ranking pairs by the margin-like ratio (minimum inter-class centroid gap over \emph{mean} within-class spread) divides by a scalar, so it cannot tell a low-variance discriminative direction from a high-variance nuisance one---and in fact prefers the latter: on a synthetic check, moving the same centroid gap onto a high-variance axis \emph{raises} the ratio from $0.27$ to $0.37$.
On MUTAG this is not hypothetical; the ratio selects degree$\times$clustering, which separates the training subsample and transfers worse ($-3.7$).
Following PALACE, we instead rank by the shrinkage-regularized Fisher--Mahalanobis margin \eqref{eq:maha}, which normalizes \emph{per direction} and collapses to $0.08$ on the same synthetic check.
We use a diagonal $\widehat S_W$---the full covariance is $K\times K$ with few samples per class and singular in practice---and $\lambda=\overline{\mathrm{var}}$.
The regularization is not incidental: at $\lambda=0.1\,\overline{\mathrm{var}}$ the criterion is under-regularized, $1/(\mathrm{var}+\lambda)$ hands large weight to coordinates whose training variance is near zero by accident, and it walks into the MUTAG trap on every fold ($-13.9$).

\subsection*{The per-prediction certificate}\label{app:pointwise}
Throughout, $x=\Psii(M)$ for a test module $M$ drawn independently of the training sample, $\delta_c=\widehat\mu_c-\mu_c$, $\sigma_c^2(x)=(x-\mu_c)^{\!\top}\Sigma_c(x-\mu_c)$ with $\Sigma_c$ the class covariance of $\Psii$, and $\kappa_c(x)=\mathbb E|\langle x-\mu_c,\Psii(M')-\mu_c\rangle|^3/\sigma_c(x)^3$ for $M'$ drawn from class $c$.

\begin{proposition}[Per-prediction agreement]\label{prop:pointwise}
Let $r$ satisfy $\Pr[\max_c\norm{\delta_c}\le r]\ge1-\delta_2$---for instance $r=r_m$ of \eqref{eq:rm} at level $\delta_2$---and let $z=\Phi^{-1}(1-\delta_1/(2C))$.
Write $\widehat c=\widehat c(M)$ and $G_{c'}(x)=\norm{x-\widehat\mu_{c'}}^2-\norm{x-\widehat\mu_{\widehat c}}^2$.
If
\[
  G_{c'}(x)\ >\ 2z\Bigl(\tfrac{\sigma_{\widehat c}(x)}{\sqrt{m_{\widehat c}}}+\tfrac{\sigma_{c'}(x)}{\sqrt{m_{c'}}}\Bigr)\ +\ r^2
  \qquad\text{for every }c'\ne\widehat c,
\]
then the population nearest-centroid rule assigns $M$ the label $\widehat c$, with probability at least
$1-\delta_1-\delta_2-0.9496\sum_c\kappa_c(x)/\sqrt{m_c}$.
Replacing $z\sigma_c(x)/\sqrt{m_c}$ by $\sqrt{2b\,\sigma_c^2(x)/m_c}+2\beta_c(x)b/(3m_c)$, where $b=\log(4C/\delta_1)$ and $\beta_c(x)$ bounds $|\langle x-\mu_c,\Psii(M')-\mu_c\rangle|$ almost surely, gives the same conclusion with no Berry--Esseen term.
\end{proposition}
\begin{proof}
By \eqref{eq:pointid} the difference between the empirical and population margins at $x$ is
$-2\ell_{c'}+2\ell_{\widehat c}+\norm{\delta_{c'}}^2-\norm{\delta_{\widehat c}}^2$ with $\ell_c=\langle x-\mu_c,\delta_c\rangle$.
Each $\ell_c$ is a mean of $m_c$ i.i.d.\ centred scalars with variance $\sigma_c^2(x)$, so the Berry--Esseen theorem with Shevtsova's constant $0.4748$, applied two-sided to each class, gives $|\ell_c|\le z\sigma_c(x)/\sqrt{m_c}$ for all $c$ simultaneously with probability at least $1-\delta_1-0.9496\sum_c\kappa_c(x)/\sqrt{m_c}$.
On that event intersected with $\{\max_c\norm{\delta_c}\le r\}$, the margin moves by at most the displayed bound, so every population margin stays positive.
The Bernstein variant replaces the first step by the scalar Bernstein inequality at level $\delta_1/(2C)$.
\end{proof}

Three limitations are worth stating exactly.
The variances $\sigma_c(x)$ and the direction $x-\mu_c$ are population quantities and are estimated in the measurements of Section~\ref{sec:exp} by $\widehat\Sigma_c$ and $x-\widehat\mu_c$, the same plug-in convention Paper~II's variance-aware radius uses; a self-contained statement needs an effective-rank bound on $\norm{\widehat\Sigma_c-\Sigma_c}_{\mathrm{op}}$ along that direction, which we do not prove here; the Hilbert-space covariance estimators and confidence balls of \citet{PaperIII} are the natural source for it.
The Berry--Esseen charge is additive and, at the class sizes of these benchmarks, not small: the plug-in $\widehat\kappa$ gives $0.09$ on NCI109 and $0.33$ on MUTAG, so the Gaussian form's nominal $1-\delta$ is reduced accordingly, and a self-normalized moderate-deviation bound---whose error is relative to the tail rather than additive---is the natural repair.
And the statement is per test point: certifying a whole test set of size $T$ simultaneously costs a $\log T$ in $z$.

\subsection*{D-COMPLEX: differentiability at ties, and the implementation}\label{app:clarke}
\begin{remark}[Where this is and is not differentiable in Proposition \ref{prop:grad}]\label{rem:clarke}
Condition (a) is the standard one: the pairing is locally constant off a closed set of measure zero in the space of filtration values, and persistence is differentiable there in the sense of \citet{Leygonie2022}, with the practical gather/scatter implementation of \citet{Carriere2021}.
Condition (b) is stronger and is \emph{not} generic; we state it separately for that reason.
A max-of-vertices filtration ties whenever two vertices of a simplex carry the same value, and on labeled molecular graphs such ties are structural, not accidental---atoms of the same type receive identical node features, so entire orbits of vertices tie exactly, on a set of positive measure in $\theta$.
What is generically unique is not the value but the \emph{argmax vertex}; at a tie $\max_v u_v$ remains locally Lipschitz but is not differentiable, and the gather rule returns one element of the Clarke subdifferential $\partial^\circ\max=\mathrm{conv}\{\partial u_v: v\in\argmax\}$ rather than a gradient.
This is the usual situation for subgradient descent on a locally Lipschitz objective, and it is what the implementation does; we call D-COMPLEX differentiable in that sense---almost everywhere, with a Clarke subgradient at the ties---and not in the sense of a globally $C^1$ map.
\end{remark}

In practice the pairing is computed once per forward pass, and the barcode endpoints are re-expressed as gathers of the differentiable tensor $s^{(s)}$; the backward pass flows through the gathers and the max-reductions to $\ftheta$, recovering the slice-stack $\Psii_\theta(M)$ as a differentiable function of the network.

\subsection*{D-COMPLEX training conditions}\label{app:dcomplex}
Gradient learning of the filtration needs two conditions or it fails outright.
The head must be the PLACE/PALACE embedding---with a lightweight top-$k$ vectorization the learned filtration does not separate from the fixed one (on DHFR both sit near $66\%$, far below the $81\%$ the same data supports), the head rather than the filtration being the bound.
And $\ftheta$ must be warm-started to reproduce the fixed filtration---otherwise its random bars miss the frozen landmarks, the embedding vanishes, and the model collapses to the majority class.

\bibliographystyle{plainnat}
\bibliography{main}

\end{document}